\documentclass[10pt]{article}

\usepackage[a4paper, total={6in,8in}]{geometry}
\usepackage{framed,multirow}
\usepackage{float}
\usepackage{latexsym}
\usepackage{amsmath,amssymb,amsfonts,graphicx}
\usepackage{enumitem}
\usepackage{url}
\usepackage{xcolor}
\definecolor{newcolor}{rgb}{.8,.349,.1}
\usepackage{lmodern}
\usepackage[scaled]{helvet}

\usepackage{tikz}
\usepackage{tikz-3dplot}
\tdplotsetmaincoords{70}{110}
\usepackage{multirow}
\usepackage{booktabs}
\usepackage{bigstrut}
\usepackage{array}
\newcolumntype{L}[1]{>{\raggedright\let\newline\\\arraybackslash\hspace{0pt}}m{#1}}
\newcolumntype{C}[1]{>{\centering\let\newline\\\arraybackslash\hspace{0pt}}m{#1}}
\newcolumntype{R}[1]{>{\raggedleft\let\newline\\\arraybackslash\hspace{0pt}}m{#1}}
\usepackage{amsthm}
\newtheorem{theorem}{Theorem}
\newtheorem*{remark}{Remark}
\usepackage{algorithm}
\usepackage{algorithmic}
\emergencystretch=\maxdimen
\newcommand{\prox}{\mathbf{prox}}

\newcommand{\argmin}{\mathop{\rm argmin}}

\newcommand{\eg}{{\it e.g.}}
\newcommand{\ie}{{\it i.e.}}

\newcommand{\vct}[1]{\boldsymbol{#1}}
\newcommand{\mtx}[1]{\boldsymbol{#1}}

\newcommand{\<}{\langle}
\renewcommand{\>}{\rangle}

\newcommand{\ve}{\vct{e}}

\newcommand{\vr}{\vct{r}}
\newcommand{\vs}{\vct{s}}

\newcommand{\vu}{\vct{u}}

\newcommand{\vw}{\vct{w}}
\newcommand{\vx}{\vct{x}}
\newcommand{\vy}{\vct{y}}
\newcommand{\vz}{\vct{z}}

\newcommand{\vtheta}{\vct{\theta}}
\newcommand{\vvarphi}{\vct{\varphi}}

\newcommand{\vzero}{\vct{0}}
\newcommand{\vone}{\vct{1}}

\newcommand{\R}{\mathbb{R}}

\newcommand{\mA}{\mtx{A}}

\newcommand{\mR}{\mtx{R}}

\newcommand{\mPsi}{\mtx{\Psi}}

\newcommand{\norm}[1]{\left\lVert#1\right\rVert}

\title{\bf CoShaRP: A Convex Program for Single-shot Tomographic Shape Sensing \footnote{The paper is under consideration at Pattern Recognition Letters}}

\author{Ajinkya Kadu$^\dagger$, Tristan van Leeuwen$^{\dagger \ddagger}$, K. Joost Batenburg$^{\dagger \mathsection}$  \\[2ex]
{\normalsize $^\dagger$ Computational Imaging, Centrum Wiskunde \& Informatica, Amsterdam, The Netherlands} \\
{\normalsize $^\ddagger$ Mathematical Institute, Utrecht University, Utrecht, The Netherlands} \\
{\normalsize $^\mathsection$ Leiden Institute of Advanced Computer Science, Leiden University, Leiden, The Netherlands}
}

\begin{document}

\maketitle

\begin{abstract}
We introduce single-shot X-ray tomography that aims to estimate the target image from a single cone-beam projection measurement. This linear inverse problem is extremely under-determined since the measurements are far fewer than the number of unknowns. Moreover, it is more challenging than conventional tomography where a sufficiently large number of projection angles forms the measurements, allowing for a simple inversion process. However, single-shot tomography becomes less severe if the target image is only composed of known shapes. Hence, the shape prior transforms a linear ill-posed image estimation problem to a non-linear problem of estimating the roto-translations of the shapes. In this paper, we circumvent the non-linearity by using a dictionary of possible roto-translations of the shapes. We  propose a convex program CoShaRP to recover the dictionary-coefficients successfully. CoShaRP relies on simplex-type constraint and can be solved quickly using a primal-dual algorithm. The numerical experiments show that CoShaRP recovers shapes stably from moderately noisy measurements. 

The code is available at \url{https://github.com/ajinkyakadu125/CoShaRP}.
\end{abstract}


\section{Introduction}
\label{sec:intro}
In tomographic imaging, the aim is to characterize the three-dimensional structure of an object from X-ray projections. In applications like medical CT, projections are gathered from all directions. This allows for a relatively straight-forward reconstruction of the object using so-called filtered back-projection methods (FBP)\cite{defrise1994cone, Kudo1998, Zou2004}. When only a limited angular sampling is available, more advanced iterative reconstruction techniques that use prior information about the structure of the object have been developed \cite{delaney1998globally,Persson2001,Sidky2008,batenburg2011dart,Frikel2013}. In this paper, we consider an extreme case where we acquire only a single X-ray projection on which to base a full three-dimensional reconstruction. We refer to this problem as \textit{single-shot X-ray tomography}. This problem is highly relevant for many applications, including industrial quality control and high-throughput imaging \cite{Hantke2014,Pelt2016,Laanait2017,Bicer2017}. A typical setup consists of a fixed X-ray source and detector that collects single X-ray projections of the objects of interest. While we specifically target reconstruction from single-angle X-ray projections, the techniques we develop are also relevant for limited-angle tomography with applications including high-resolution dynamic imaging\cite{Spence2012,Zhu2020} and cryo-electron microscopy \cite{Zhang2008,Scheres2012}.

\begin{figure}[!htb]
    \centering
    \begin{tabular}{c c c}
        (a) True & (b) Projection & (c) FBP \\
        \includegraphics[width=0.2\columnwidth]{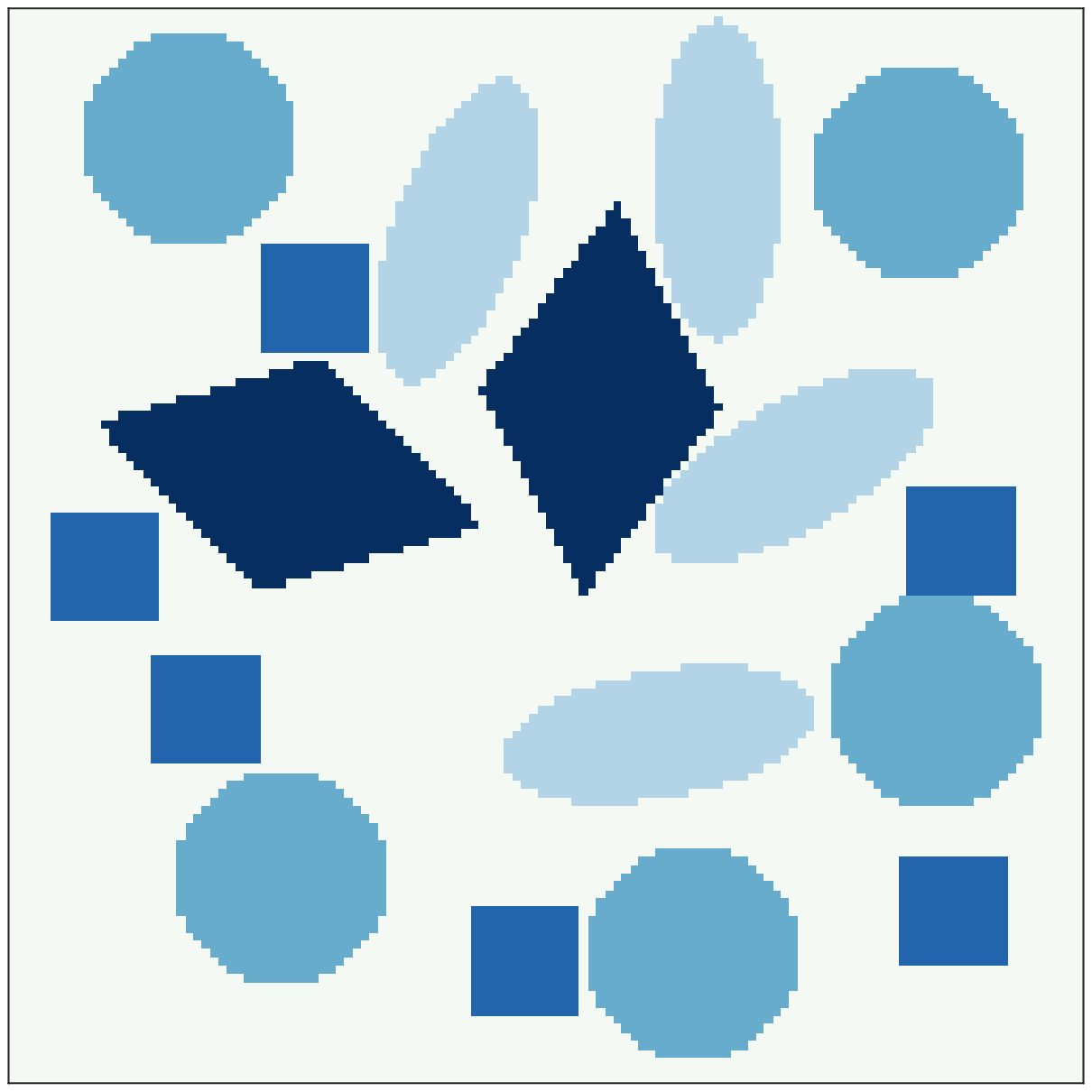} & \includegraphics[width=0.2\columnwidth]{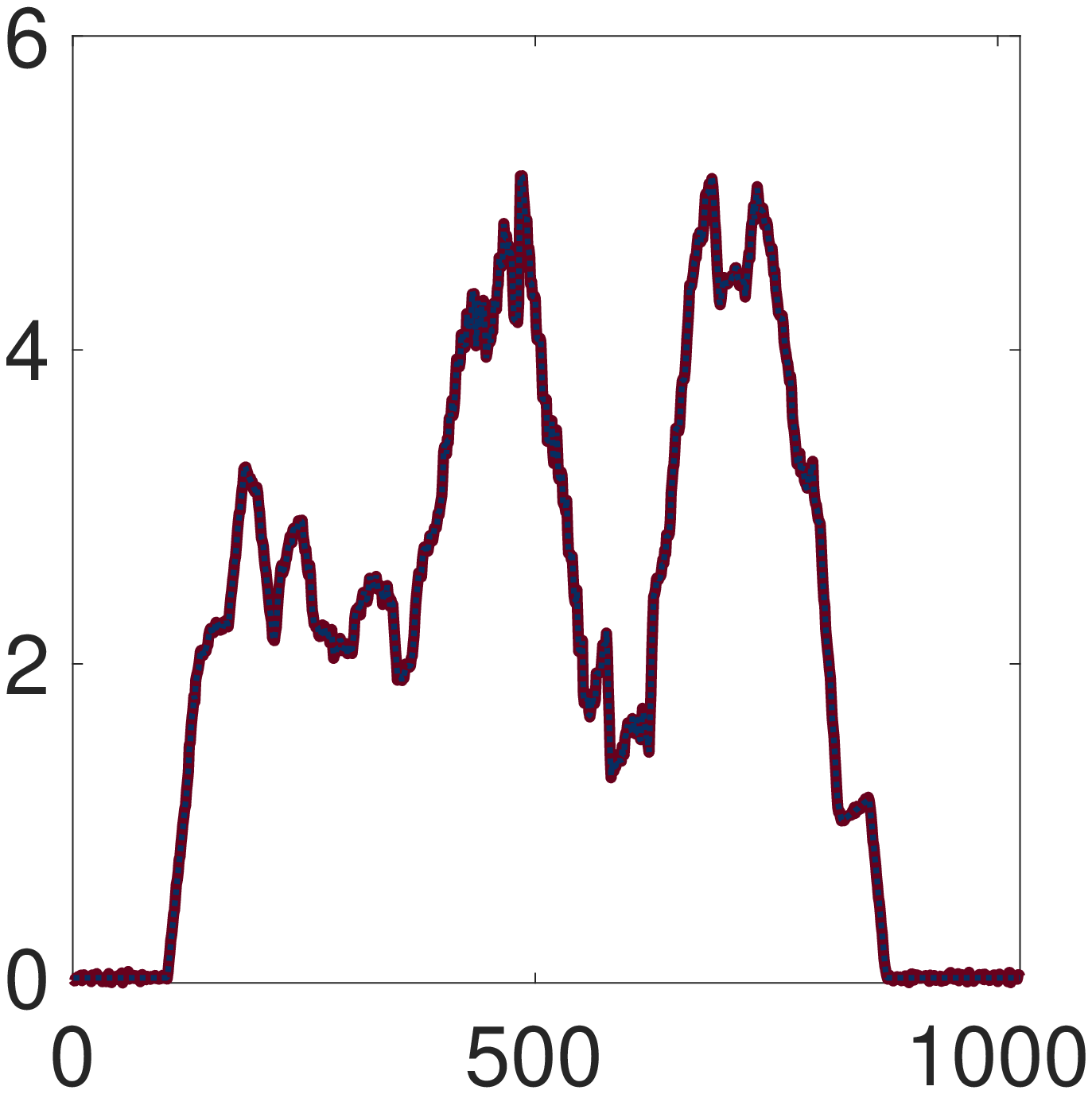} & \includegraphics[width=0.2\columnwidth]{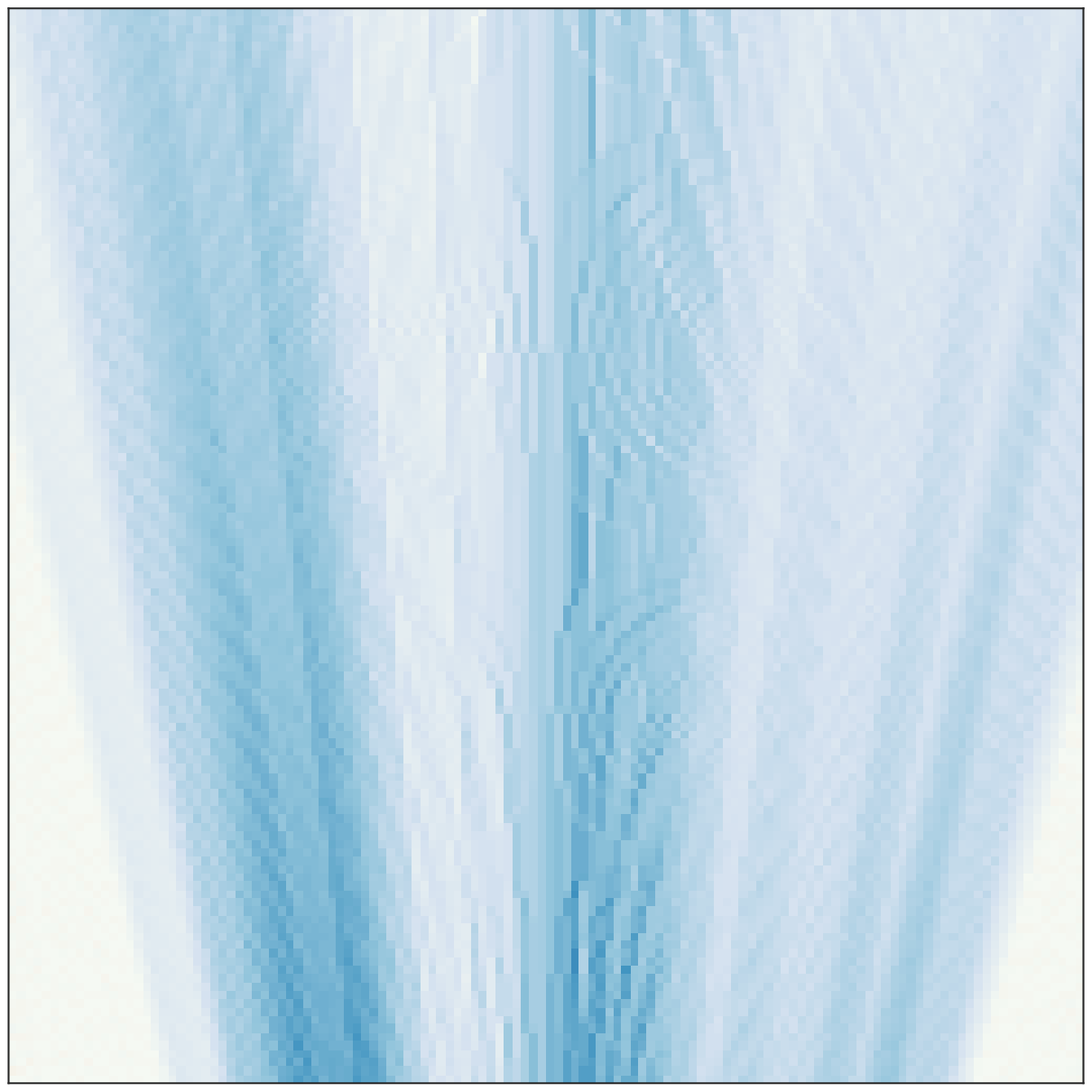} \\[1ex]
        (d) TV & (e) SSC & (f) CoShaRP \\
        \includegraphics[width=0.2\columnwidth]{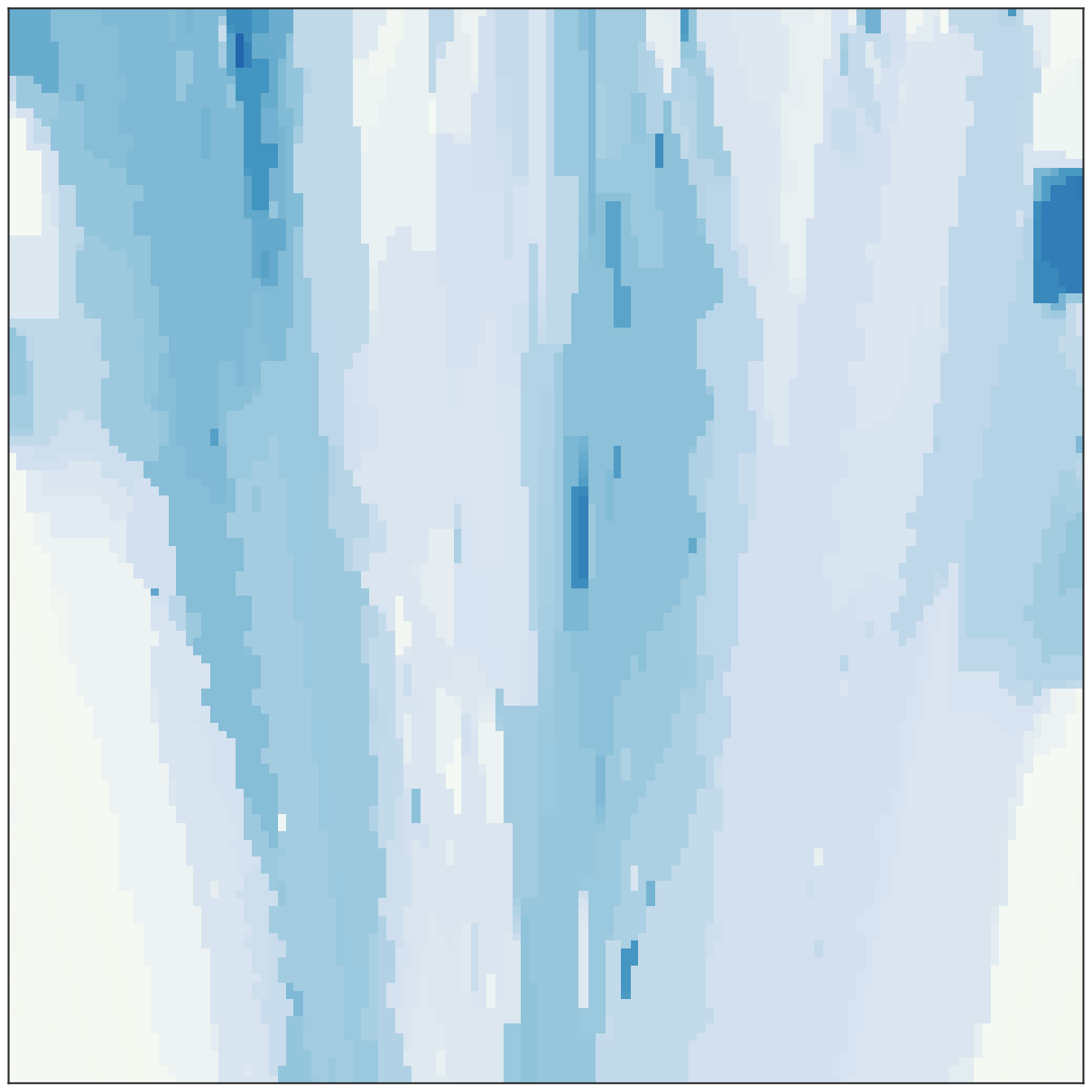} & \includegraphics[width=0.2\columnwidth]{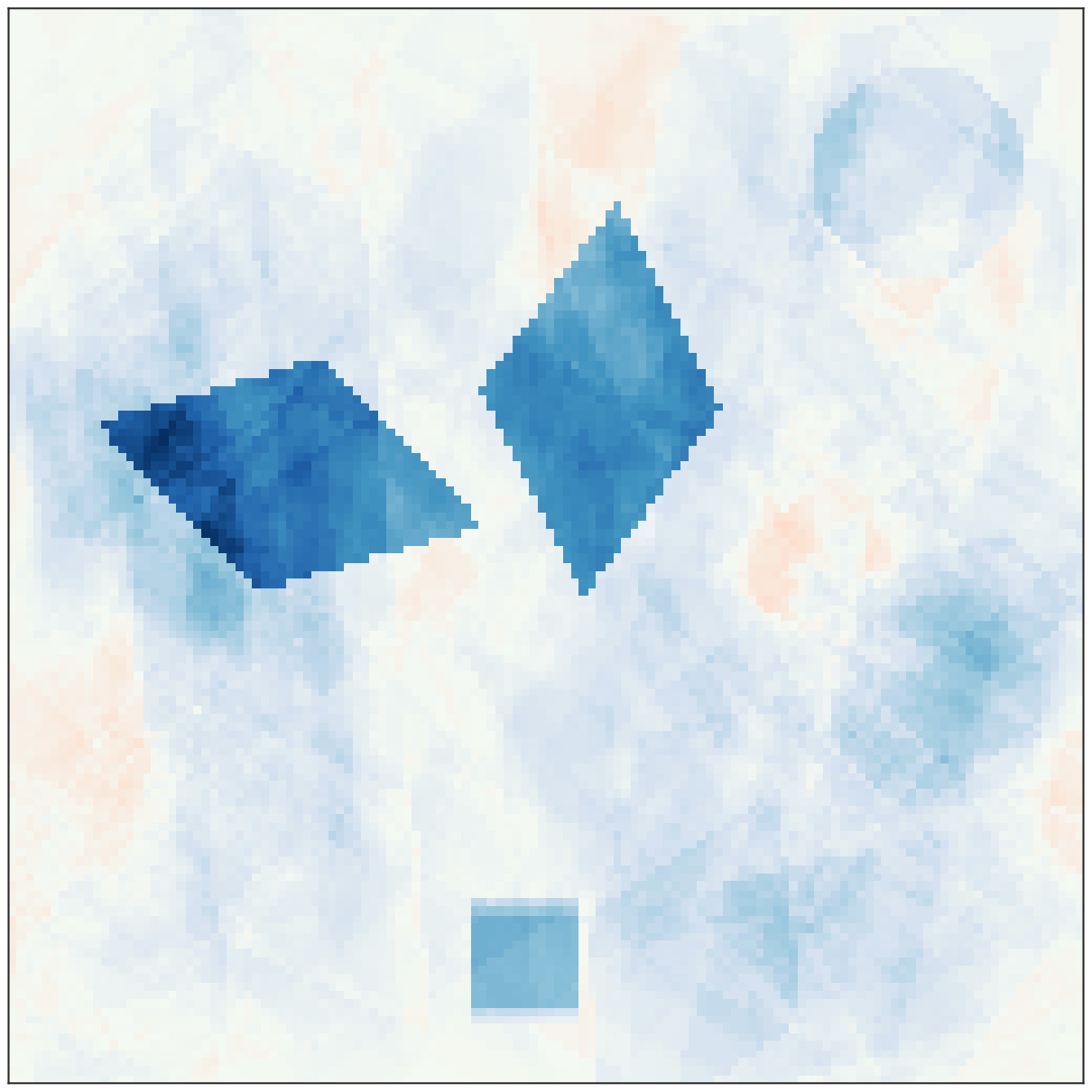} &  \includegraphics[width=0.2\columnwidth]{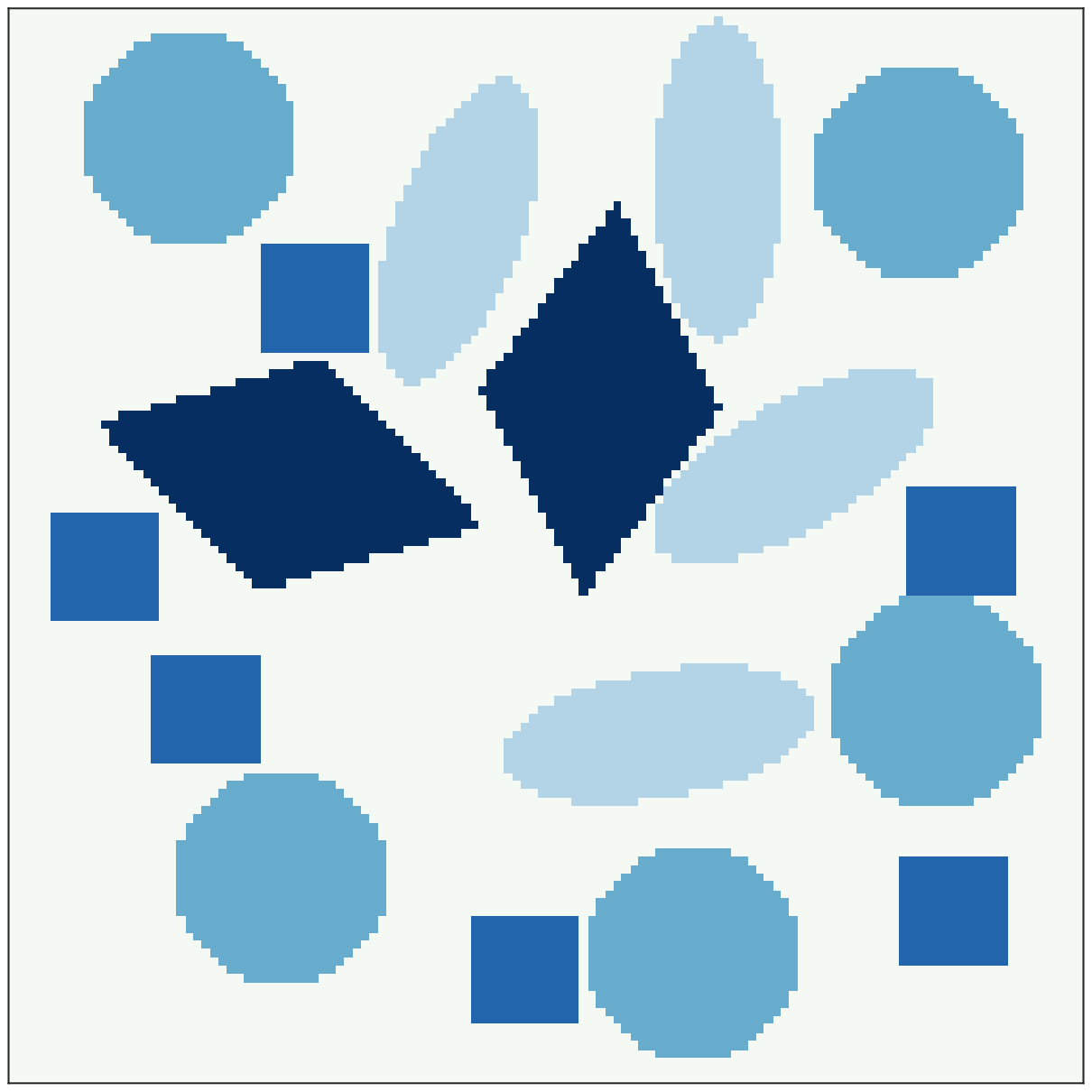} \\
        & \includegraphics[width=0.2\columnwidth]{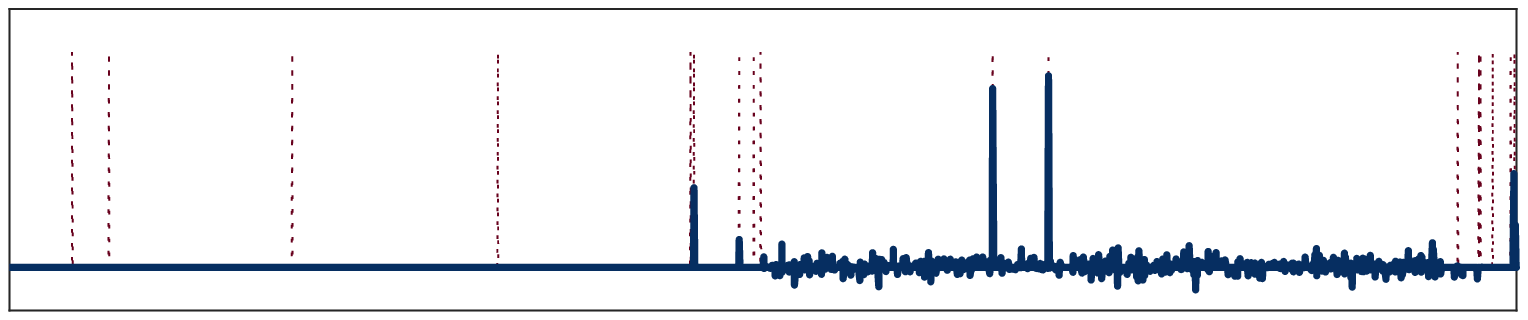} & \includegraphics[width=0.2\columnwidth]{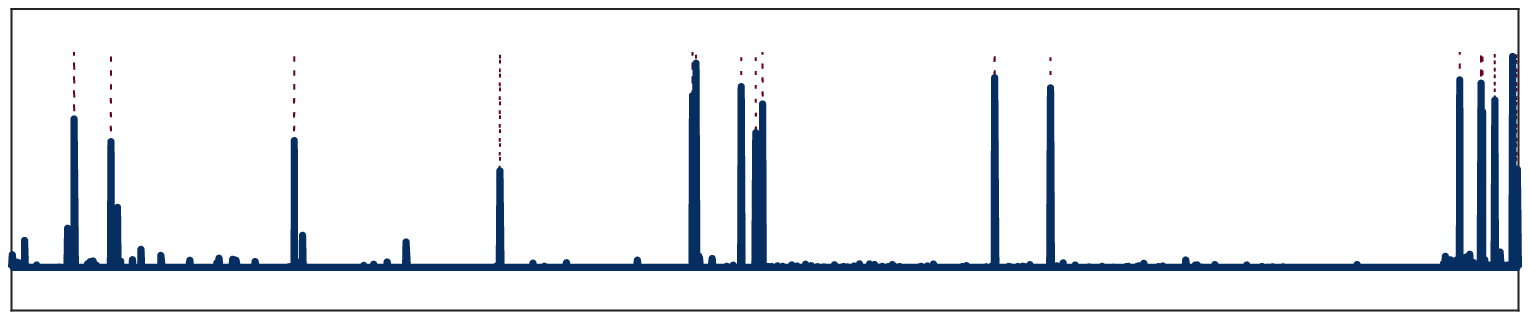}
    \end{tabular}
    \caption{Demonstration of single-shot X-ray tomography. A target image (a) of size $1~m \times 1~m$ discretized on $128 \times 128$ pixel grid has a fan-beam projection shown in (c). The image consists of 4 different shapes with different intensities. The equally-spaced detectors place at the top of the image collect a total of $1024$ measurements resulting from an X-ray source at the bottom of the image. We consider FBP (c), Total Variation (TV) regularization (d), and Sparse Shape Composition (e). Our proposed approach CoShaRP is given in (f). The shape coefficients for SSC and CoShaRP are given below their figures, while dotted ones denote the correct coefficients. \vspace{-2mm}}
    \label{fig:intro}
\end{figure}

The single-shot X-ray tomography problem is extremely under-determined, making it an ill-posed inverse problem. A single cone-beam projection of a volume containing $n^3$ voxels consists of $\mathcal{O}(n^2)$ measurements. Hence, the measurements are undersampled by a factor of $n$. To reduce the ill-posedness, it is a general practice to incorporate prior information via regularization. However, classic regularization methods fail on single-shot X-ray tomography, as demonstrated in Fig.~\ref{fig:intro}.

Since it is evident that a strong prior is necessary to recover the target image from a single projection, we consider the class of objects that are composed of a limited number of known building blocks. This is a reasonable assumption when imaging materials that are made up of basic structures, for example, a 3D structure comprising of interlocking bricks or a protein structure consisting of repeated amino acid groups. Hence, if such shapes are known a-priori along with their number of repetitions, the image estimation problem can be recast as an estimation of roto-translation parameters of these shapes. We term the process of estimating the shape parameters (\ie, their roto-translation in the image) from the linear measurements of the image that constitutes them as \textit{shape sensing}. Unlike the image estimation problem -- which has a simple linear structure -- the estimation of roto-translation parameters (\ie, shape sensing) is a non-linear problem. This, in turn, makes the inversion process computationally intractable. To avoid such non-linearity, we use a shape dictionary approach that expresses the target image as a linear combination of shapes from the available dictionary. Due to the linear structure, we show that it is possible to recover the shapes from a single-shot X-ray projection (see Fig.~\ref{fig:intro}) by solving a convex problem. We regard the recovery of shapes from their tomographic projections as \emph{tomographic shape sensing}.

\subsection{Related Work}

The shape sensing problem has been studied in the context of shape-based characterization, object tracking and optical character recognition. Inspired by the compressed sensing, a recent approach called Sparse Shape Composition (SSC) imposes an $\ell_1$-norm constraint on the shape-dictionary coefficients  \cite{Aghasi2013, Aghasi2015}. SSC has the advantage that it can form new shapes from the intersection and union of basic shapes. However, the main drawback comes from large feasible solution space inherent to the $\ell_1$-norm constraint in high dimension (see Remark~\ref{remark:CoShaRPvsSSC}). This large feasible space may lead to an incorrect solution. A simplistic version of SSC performs the 3D characterization of nano-particles using electron tomography \cite{Zanaga2016}. This method uses a simple $\ell_1$-norm constraint to recover spherical nano-particles from their tomographic projections. However, their tomographic projections have a parallel-beam geometry and require measurements from more than one projection angle. Although SSC has been extended to tomographic inverse problem, it has never been tested for single-shot tomographic shape sensing. However, we demonstrate the failure of SSC in single-shot cone-beam tomography in Fig.~\ref{fig:intro}.

\subsection{Contributions and Outline} 
To the best of our knowledge, the single-shot X-ray tomography has never been studied, and no reconstruction method exist till date to recover back an image successfully from a single-shot. We introduce the tomographic shape sensing problem that assumes the prior knowledge about the shapes in the image. The principal contribution of this paper is the development of the convex program CoShaRP to reconstruct images composed of non-overlapping shapes from a single-shot. The convex program is novel in the sense that the simplex-type constraint enables sharp recovery results from extremely under-determined single-shot tomographic projections. Although the exact recovery problem is NP-hard, our proposed convex program CoShaRP stably recovers the shapes. Moreover, we propose a primal-dual algorithm to find the optimal solution of CoShaRP. The algorithm does not rely on the inversion of large matrices and uses a simple proximal operator to project onto a $K-$simplex constraint. Using numerical experiments, we answer the following questions:
\begin{enumerate}[topsep=5pt,itemsep=0ex,partopsep=1ex,parsep=1ex]
    \item What is the minimum resolution of the shape that CoShaRP can sense?
    \item Is CoShaRP robust to the number of repetitions and the possible rotations of the shapes in the target image?
    \item Can CoShaRP recover non-homogeneous as well as non-convex shapes?
    \item How sensitive is CoShaRP to the measurement noise?
\end{enumerate}
We discuss the single-shot tomographic inverse problem in Section~\ref{sec:tomography}. Section~\ref{sec:CSR} discusses the tomographic shape sensing problem and introduces a convex program CoShaRP. We describe an efficient iterative scheme to find an optimal solution to CoShaRP in Section~\ref{sec:Opt}. We illustrate the numerical experiments in Section~\ref{sec:Results} and conclude the paper in Section~\ref{sec:Conclusion}. 

\subsection{Notation}
Throughout this paper, small boldface letters (\eg, $\vx, \vz$) denote vectors in $\R^n$. The identity and zero elements are denoted by $\vzero$ and $\vone$ respectively. The Euclidean inner product is denoted by $\< \vx, \vy \> = \sum\nolimits_{i=1}^{n} x_i y_i$ for $\vx, \vy \in \R^n$ with a corresponding norm $\| \vx \| = \sqrt{\< \vx, \vx \>}$. However, for all other norms, we use subscripts (\eg, $\| \vx \|_1 = \sum_{i=1}^n |x_i|$, $\|\vx \|_{\infty} = \max_i |x_i|$). To represent the matrices, we use uppercase letters (\eg, $\mA, \mPsi$). The elements of a matrix $\mA$ are denoted by $a_{ij}$. All the functions are represented as $f: \mathcal{X} \mapsto \mathcal{Y}$, where $\mathcal{X}$ and $\mathcal{Y}$ are the domain and co-domain of $f$, respectively. We denote the convex conjugate of the function $f$ by $f^\star$. $\prox_{f}( \vz)$ denotes the proximal of function $f$ evaluated at point $\vz$  (for definition, please refer to \cite{Combettes2011}). We represent the optimal solution to the optimization problem using overline (\eg, $\overline{\vx}, \overline{{\mu}}$).


\section{Single-shot X-ray Tomography}
\label{sec:tomography}

The acquisition geometry of single-shot X-ray tomography consists of one source and an array of regularly spaced detectors. Let $\vvarphi \in \mathbb{S}^{d-1}$ be a directional vector, and $\vr \in \R^d$ be any position vector, where $d \in \{2,3\}$ is the dimension of the scene. The cone-beam transform $A_C$ of an image function $x: \R^d \mapsto \R$ is its integral along a line in the direction $\vvarphi$ passing through $\vr$. It is mathematically given by
\begin{equation*}
    (A_C x)(\vr,\vvarphi) = \int_{0}^{\infty} x(\vr + t \vvarphi) \, \mathrm{d} t.
\end{equation*}
In a single-shot setup, we have a source located at $\vr_0$. It sends multiple X-rays through the object (compactly supported on $\Omega \subset \R^d$) in a cone with a vertex at $\vr_0$ and spanning angles in a set $\Phi$ that determines the geometry of cone. Let these angles be $\vvarphi_i \in \Phi$, $i=1, \dots, m$, then the measurement $y_i$ is
\begin{align*}
    y_i &= \left(A_C x \right) \left(\vr_0,\vvarphi_i \right) 
    \approx \sum_{j=1}^{n} a_{ij} x_j,
\end{align*}
where $x_j$ denotes the value of $x$ in the $j^\text{th}$ voxel and $a_{ij}$ is the contribution of the $j^\text{th}$ voxel to the $i^\text{th}$ ray.
The measurements can now be expressed as a linear system of equations
\[
\vy = \mA \vx.
\]
The above linear system of equations is extremely under-determined since the number of measurements $m$ is far smaller than the number of unknowns $n$. We do assume that each voxel is intersected by at least one ray, so that each column of the matrix has at least one non-zero element. Determining the image from the measurements is an ill-posed inverse problem. In general, to resolve this ill-posedness, regularization needs to be added in the inversion procedure to incorporate the prior knowledge about the target image. However, conventional regularization techniques are not sufficient to resolve the true image, as was illustrated in Fig.~\ref{fig:intro}.


\section{Convex Shape Recovery}
\label{sec:CSR}

In this section, we discuss the tomographic shape sensing problem. Our formulation hinges on the formation of a dictionary that consists of possible roto-translations of the known shapes and the representation of target image as a convex combination of dictionary elements. 

\subsection{Image model and Dictionary}
Let the functions $u_i: \Omega \mapsto \R$, $i=1, \dots, S$, denote the compactly supported shape functions. The image is now assumed to be composed of roto-translations of these shapes
\[
    x(\vr) = \sum_{i=1}^{S} \sum_{j=1}^{k_i} u_i \left( \mR\left(\vtheta_{i,j}\right) \vr + \vs_{i,j} \right)
\]
where $\vtheta_{i,j} \in \R^{d(d-1)/2}$ and $\vs_{i,j} \in \R^d$ are the angle and the shift of $j^{\text{th}}$ copy of shape $i$ respectively, and $\mR \in \R^{d \times d}$ is a rotation matrix that depends on the angle $\vtheta$. The total number of shapes in an image are $K = k_1 + \dots + k_S$. 
Hence, from the knowledge of the shapes, the image estimation translates to finding the roto-translation parameters $(\vtheta, \vs)$ of the shapes. However, the image is a non-linear function of these parameters. Hence, the recovery of these parameters becomes a computationally intractable problem due to the non-convex structure of the cost function.


To mitigate the non-linearity associated with the roto-translation parameters, we create a \textit{shape} dictionary that consists of roto-translations of the shapes. Let the dictionary 
\begin{align*}
    \mPsi (\vr) &= \left[ \widehat{\mPsi}_1 (\vr), \dots, \widehat{\mPsi}_S (\vr) \right], \\[1ex]
    \mbox{with} \quad \widehat{\mPsi}_i (\vr) &= \big[ u_i \left(  \mR \left( \vtheta_{i,j} \right) \vr  + \vs_{i,j} \right) \big]_{j=1}^{J}, \qquad i = 1, \dots, S,
\end{align*}
where $j=1, \dots, J$ covers possible roto-translations of the shapes. Hence, the target image can be represented as a linear combination of the elements of this dictionary,
\begin{align*}
    x(\vr) &= \sum\nolimits_{i=1}^{p} z_i \psi_i (\mathbf{r}) \\[1ex]
    \mbox{with} & \quad  z_i \in \{0, 1 \}, \quad i=1, \dots, p,
\end{align*}
where $\vz = \left[ z_1, \dots, z_p \right]^T$ is a coefficient vector, with $p = JS$. Hence, the shape recovery problem is to find a high-dimensional binary vector $\vz$ from its linear measurements
\[
\vy = \mA \mPsi \vz.
\]
Here, $\mA \mPsi$ contains the projections of the individual dictionary elements, sampled at the appropriate points. The binary constraints on $\vz$ make the recovery problem an integer program, and hence, NP-hard in general \cite{Karp1972}.

\subsection{Convex Shape Recovery Program (CoShaRP)}
The binary constraints on the coefficients can be relaxed using the bounds constraints. Moreover, a Gaussian assumption on the noise leads to a least-squares formulation for the data misfit. Hence, the resulting convex program, which we refer to as the \emph{Convex Shape Recovery Program} (CoShaRP), reads
\begin{equation}
    \begin{split}
        \underset{\vz \in \R^p}{\mbox{minimize}} & \quad \norm{ \mA \mPsi \vz - \vy } \\
        \mbox{subject to} & \quad  \vz^T \vone = K, \quad \vzero \leq \vz \leq \vone.
    \end{split}
    \label{eq:CoShaRP}
\end{equation}
Here, the inequality between vectors is imposed elementwise. Note that we have used the Euclidean norm instead of its square to measure the misfit. 

The geometric interpretation of CoShaRP is as follows: We are trying to find a high-dimensional vector $\vz$ closest to the hyperplane $\mA \mPsi \vz = \vy$ in a Euclidean sense that lies on the intersection of a hyperplane $\vz^T \vone = K$ and the hyperbox $\vzero \leq \vz \leq \vone$. In Fig.~\ref{fig:geometry}, we show the geometry for a shape-sensing problem with two possible shapes ($p=2$). 
In Fig.~\ref{fig:geometry}(a), the hyperplane corresponding to tomographic measurements intersects the hyperplane corresponds to equality constraints ($\vz^T \vone = K$) outside the hyperbox. Hence, the solution to CoShaRP in this case is binary. However, a binary solution can not always be guaranteed as these hyperplanes may intersect inside the hyperbox (cf. Fig.~\ref{fig:geometry}(b)). In such cases, further post-processing 
is required to retrieve the target image. For more, refer to Section~\ref{sec:Opt:ImForm}.

The CoShaRP consists of constraints that are defined by K-simplex. The K-simplex, defined as
\[
    \Delta_{p}^K = \left\lbrace \vz \in \R^p \, | \, \sum\nolimits_{i=1}^{p} z_i = K, \, \vzero \leq \vz \leq \vone \right\rbrace ,
\]
is a generalized version simplex (simplex has $K=1$). K-simplex represents a polytope in $p$-dimension with $\binom{p}{k}$ vertices. Moreover, these polytopes are \textit{regular}, \ie, they posses highest level of symmetry. We plot $K$-simplex in three dimension in Fig.~\ref{fig:3DSimplex}. These simplices are equilateral triangle except for $K=3$. However, it is important to note that the number of shapes in the target image will be much smaller than the number of dictionary elements (\ie, $K \ll p$). Hence, we will frequently encounter feasible regions to be extremely low-dimensional polytope embedded in a high-dimensional space.

\begin{figure}[t]
    \centering
    \begin{tabular}{c c}
        (a)  & (b) \\
        \resizebox{0.45\columnwidth}{!}{
        \begin{tikzpicture}[scale=0.6]
        \def\laxis{5}
        \def\ltr{3}
        \begin{scope}[->,black]
            \draw (0,0) -- (\laxis,0) node [below] {$z_1$};
            \draw (0,0) -- (0,\laxis) node [right] {$z_2$};
        \end{scope}
        \node [] at (\ltr,-0.5) {1};
        \node [] at (-0.5,\ltr) {1};
        \filldraw [opacity=0.25,gray] (0,0) -- (0,\ltr) -- (\ltr,\ltr) -- (\ltr,0) -- cycle;
        \draw[black] (-0.5*\ltr,1.5*\ltr) -- (1.5*\ltr,-0.5*\ltr) node [anchor=south east] {$\vz^T \vone = K$};
        \draw[opacity=0.5,red,very thick] (0,\ltr) -- (\ltr,0);
        \draw[thick,dashed] (-0.3*\ltr,1.6*\ltr) -- (1.5*\ltr,0.2*\ltr) node [anchor=south west] {$\mA \mPsi \vz = \vy$};
        \node [opacity=0.5,green!40!black,thick] at (0,\ltr) {$\bigstar$};
        \end{tikzpicture} 
        }
        &
        \resizebox{0.45\columnwidth}{!}{
        \begin{tikzpicture}[scale=0.6]
        \def\laxis{5}
        \def\ltr{3}
        \begin{scope}[->,black]
            \draw (0,0) -- (\laxis,0) node [below] {$z_1$};
            \draw (0,0) -- (0,\laxis) node [right] {$z_2$};
        \end{scope}
        \node [] at (\ltr,-0.5) {1};
        \node [] at (-0.5,\ltr) {1};
        \filldraw [opacity=0.25,gray] (0,0) -- (0,\ltr) -- (\ltr,\ltr) -- (\ltr,0) -- cycle;
        \draw[black] (-0.5*\ltr,1.5*\ltr) -- (1.5*\ltr,-0.5*\ltr) node [anchor=south west] {$\vz^T \vone = K$};
        \draw[opacity=0.5,red,very thick] (0,\ltr) -- (\ltr,0);
        \draw[thick,dashed] (0.5*\ltr,-0.3*\ltr) -- (0.2*\ltr,1.5*\ltr) node [anchor=north west] {$\mA \mPsi \vz = \vy$};
        \node [opacity=0.5,green!40!black,thick] at (0.35*\ltr,0.65*\ltr) {$\bigstar$};
        \end{tikzpicture} 
        }
    \end{tabular}
    \caption{Geometry of CoShaRP. The grey region denotes the hyperbox that corresponds to bounds constraints($\vzero \leq \vz \leq \vone$). Solid line denotes the hyperplane for constraint on the number of shapes in the image, while dotted line denotes the hyperplane for measurement $\mA \mPsi \vz = \vy$. Note that the measurement hyperplanes do not pass through point $(0,1)$ or $(1,0)$ due to noise in the measurements. The star denotes the solution of CoShaRP. The figure (a) denotes the setup where coefficient is binary, while (b) with a non-binary solution.} \vspace{-4mm}
    \label{fig:geometry}
\end{figure}
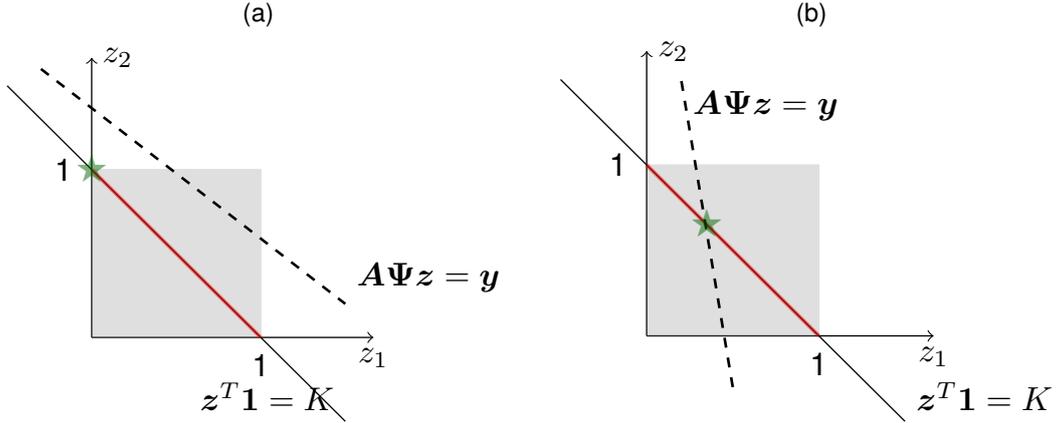

\begin{figure}[!b]
    \centering
    \begin{tabular}{c c c}
        (a) 1-simplex & (b) 2-simplex & (c) 3-simplex \\
        \resizebox{0.3\columnwidth}{!}{
        \begin{tikzpicture}[tdplot_main_coords,scale=2]
        \def\laxis{4}
        \def\ltr{3}
        \begin{scope}[->,black]
            \draw (0,0,0) -- (\laxis,0,0); 
            \draw (0,0,0) -- (0,\laxis,0); 
            \draw (0,0,0) -- (0,0,\laxis); 
        \end{scope}
        \filldraw [opacity=0.5,red] (\ltr,0,0) -- (0,\ltr,0) -- (0,0,\ltr) -- cycle;
        \filldraw [opacity=.2,gray] (\ltr,0,0) -- (\ltr,0,\ltr) -- (\ltr,\ltr,\ltr) -- (\ltr,\ltr,0) -- cycle;
        \filldraw [opacity=.2,gray] (0,\ltr,0) -- (\ltr,\ltr,0) -- (\ltr,\ltr,\ltr) -- (0,\ltr,\ltr) -- cycle;
        \filldraw [opacity=.2,gray] (0,0,\ltr) -- (\ltr,0,\ltr) -- (\ltr,\ltr,\ltr) -- (0,\ltr,\ltr) -- cycle;
        \end{tikzpicture} 
        }
        &
        \resizebox{0.3\columnwidth}{!}{
        \begin{tikzpicture}[tdplot_main_coords,scale=2]
        \def\laxis{4}
        \def\ltr{3}
        \begin{scope}[->,black]
            \draw (0,0,0) -- (\laxis,0,0); 
            \draw (0,0,0) -- (0,\laxis,0); 
            \draw (0,0,0) -- (0,0,\laxis); 
        \end{scope}
        \filldraw [opacity=.5,red] (\ltr,\ltr,0) -- (0,\ltr,\ltr) -- (\ltr,0,\ltr) -- cycle;
        \filldraw [opacity=.2,gray] (\ltr,0,0) -- (\ltr,0,\ltr) -- (\ltr,\ltr,\ltr) -- (\ltr,\ltr,0) -- cycle;
        \filldraw [opacity=.2,gray] (0,\ltr,0) -- (\ltr,\ltr,0) -- (\ltr,\ltr,\ltr) -- (0,\ltr,\ltr) -- cycle;
        \filldraw [opacity=.2,gray] (0,0,\ltr) -- (\ltr,0,\ltr) -- (\ltr,\ltr,\ltr) -- (0,\ltr,\ltr) -- cycle;
        \end{tikzpicture}
        }
        &
        \resizebox{0.3\columnwidth}{!}{
        \begin{tikzpicture}[tdplot_main_coords,scale=2]
        \def\laxis{4}
        \def\ltr{3}
        \begin{scope}[->,black]
            \draw (0,0,0) -- (\laxis,0,0); 
            \draw (0,0,0) -- (0,\laxis,0); 
            \draw (0,0,0) -- (0,0,\laxis); 
        \end{scope}
        \node [opacity=0.5,red,line width=20pt] at (\ltr,\ltr,\ltr) {\scalebox{2}{$\bullet$}};
        \filldraw [opacity=.2,gray] (\ltr,0,0) -- (\ltr,0,\ltr) -- (\ltr,\ltr,\ltr) -- (\ltr,\ltr,0) -- cycle;
        \filldraw [opacity=.2,gray] (0,\ltr,0) -- (\ltr,\ltr,0) -- (\ltr,\ltr,\ltr) -- (0,\ltr,\ltr) -- cycle;
        \filldraw [opacity=.2,gray] (0,0,\ltr) -- (\ltr,0,\ltr) -- (\ltr,\ltr,\ltr) -- (0,\ltr,\ltr) -- cycle;
        \end{tikzpicture}
        }
    \end{tabular}
    \caption{K-simplex in 3D. Barring $K=3$ case where the K-simplex reduces to a point, the simplex are equilateral triangles (denoted by red color). The gray box denotes the bounds constraints.} 
    \label{fig:3DSimplex}
\end{figure}
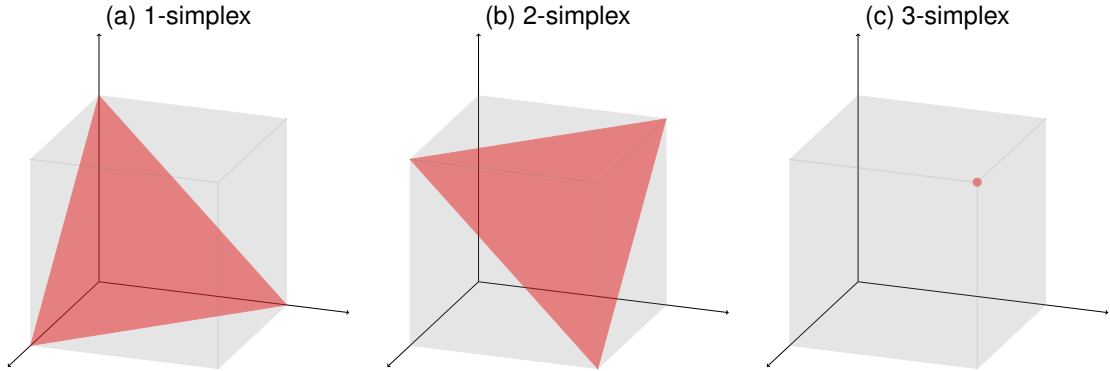

\begin{remark}
\label{remark:CoShaRPvsSSC}
CoShaRP differs significantly from the Sparse Shape Composition (SSC) \cite{Aghasi2013, Zanaga2016}. SSC formulates the shape-sensing problem as
\begin{align}
    \mbox{minimize} \quad \tfrac{1}{2} \norm{ \mA \mPsi \vz - \vy}^2 \quad \mbox{subject to} \quad \norm{ \vz }_1 \leq K.
    \label{eq:SSC}
\end{align}
The $\ell_1$-norm ball is bigger in size than the K-simplex constraint set used in CoShaRP. In particular, the K-simplex constraint represents the strict $\ell_1$ ball, \ie, $\norm{\vz}_1 = K$, in the non-negative region. Since the solution lies on the corners of the K-simplex constraint, the recovery with CoShaRP is stronger than that of SSC (see, \eg, Fig.~\ref{fig:intro}).
\end{remark}

\section{Optimization}
\label{sec:Opt}

We discuss a fast iterative scheme to find an approximate solution of CoShaRP numerically. The iterative scheme is based on splitting strategy that separates the non-smooth part from the smooth. We also introduce a thresholding method to recover the image from the coefficient vector, in case the solution is not binary.

\subsection{Primal-Dual Algorithm}
For simplicity, we express CoShaRP in the following form:
\begin{align*}
    \mbox{minimize} \quad & f(\mA \mPsi \vz) + g(\vz), \\[1ex]
    \mbox{where} \quad  f(\vz) &= \norm{\vz - \vy }, \quad g(\vz) = \delta_\mathcal{C}(\vz), \\
    \mbox{with} \quad  \mathcal{C}(\vz) &= \Big\lbrace \vz \in \R^n \, | \, \vz^T \vone = K, \, \vzero \leq \vz \leq \vone \Big\rbrace,
\end{align*}
and $\delta_{\mathcal{C}}$ is the indicator function of the set $\mathcal{C}$. 
To solve this optimization problem, we use a primal-dual splitting algorithm \cite{Zhang2010, Chambolle2010}. The iterates of this primal-dual algorithm takes the following form:
\begin{align*}
    \vz_{t+1} &= \prox_{\gamma g} \Big( \vz_t - \gamma \mPsi^T \mA^T \vu_t \Big), \\
    \vu_{t+1} &= \prox_{\tau f^\star} \Big( \vu_t - \tau \mA \mPsi \left(\vz_t - 2 \vz_{t+1} \right) \Big),
\end{align*}
for $t = [0,1,\dots,T]$, where $\gamma, \tau > 0$, with $ \gamma \tau \leq \| \mA \mPsi \|^{-1}$, are parameters that controls the speed of convergence. 
The main characteristic of this algorithm is that we avoid an inversion of a large matrix $\mPsi$ which often occurs in other splitting methods such as alternating direction method of multipliers \cite{Boyd2010}. Moreover, the proximal of both functions are easy to compute.

\begin{algorithm}[t]
    \caption{Primal-Dual Algorithm for CoShaRP}
    \label{alg:split}
    \begin{algorithmic}[1]
    \renewcommand{\algorithmicrequire}{\textbf{Input:}}
    \renewcommand{\algorithmicensure}{\textbf{Output:}}
    \REQUIRE $\mA \in \R^{m \times n}, \mPsi \in \R^{n \times p}, \vy \in \R^{m}, \gamma, \tau, T, \epsilon$
    \ENSURE  $\overline{\vz}$
 \\ \textit{Initialisation} : $\vz_0, \vu_0$
  \FOR {$t = 0$ to $T$}
  \STATE compute $\vz_{t+1} = \! \prox_{\gamma g} \! \! \left(\vz_t - \gamma \mPsi^T \! \mA^T \! \vu_t \right)$ using eq.~\eqref{eq:OrthProj:Formula}
  \STATE compute $\vu_{t+1} = \prox_{\tau f^\star} \! \! \left(\vu_t - \tau \mA \mPsi (\vz_t - 2\vz_{t+1}) \right)$ using eq.~\eqref{eq:Prox:fstar}
  \IF {$\| \mA \mPsi \vz_t - \vy \| \leq \epsilon$}
    \RETURN $\overline{\vz} = \vz_{t+1}$
  \ENDIF
  \ENDFOR
 \RETURN $\overline{\vz} = \vz_{T}$
 \end{algorithmic}
\end{algorithm}

\begin{remark}
The proposed primal-dual algorithm does not require user to store a large dictionary matrix $\mPsi$ as well as tomography matrix $\mA$. Hence, the algorithm can utilize the functional forms of the dictionary as well as tomography operator since it only requires the forward and the adjoint operation with the operator. 
\end{remark}

\subsection{Proximal operators}
The conjugate function of $f(\vx) = \norm{ \vx - \vy } $ is
\begin{align*}
    f^\star(\vw) &= - \inf_{\vx} \big\lbrace f(\vx) -  \vx^T \vw \big\rbrace \\
    &= \begin{cases}
    \vw^T \vy & \quad \mbox{if} \quad \| \vw \| \leq 1 \\
    + \infty & \quad \mbox{otherwise}
    \end{cases}
\end{align*}
We refer to \cite[Example~3.26]{boyd2004convex} for the derivation of the conjugate function. The conjugate function is linear inside the Euclidean norm ball of size 1 and $\infty$ outside. Hence, the conjugate function is convex. Its proximal operator is given in the following theorem. 
\begin{theorem} 
The proximal operator of function 
\[
    f^\star(\vx) = \begin{cases}
    \vy^T \vx & \quad \| \vx \| \leq 1\\
    + \infty & \quad \mbox{otherwise}
    \end{cases},
\]
where $\vy \in \R^n$ is a known vector, is given by
\begin{equation}
    \prox_{\gamma f^\star}(\vz) = \frac{\vz - \gamma \vy}{\max \left(1, \| \vz - \gamma \vy \| \right)}.
    \label{eq:Prox:fstar}
\end{equation}
\end{theorem}
\begin{proof}
    The proximal operator for function $f^\star$ reads
    \begin{align*}
        \prox_{\gamma f^\star}\! \! (\vz) &= \underset{\vx}{\argmin} \left\lbrace \frac{1}{2 \gamma} \| \vx - \vz \|^2 + \vx^T \vy : \| \vx \| \leq 1 \right\rbrace, \\
        &= \underset{\| \vx \| \leq 1}{\argmin} \bigg\lbrace \| \vx - \left( \vz - \gamma \vy \right) \|^2 - \gamma^2 \| \vy \|^2 + 2 \gamma \vz^T \vy \bigg\rbrace .
    \end{align*}
    The Euclidean norm constraints enforces two cases: (\emph{i}) The optimal point without the constraints is $\overline{\vx} = \vz - \gamma \vy$. This optimal solution holds when $\| \overline{\vx} \| \leq 1$. (\emph{ii}) When $\| \vz - \gamma \vy \| > 1$, the optimal solution lies on the surface of the Euclidean norm ball with size 1. Moreover, the optimal solution is in the direction of $\vz - \gamma \vy$. Hence, the proximal point is
    \[
        \prox_{\gamma f^\star}(\vz) = \begin{cases}
        \vz - \gamma \vy & \quad \| \vz - \gamma \vy \| \leq 1 \\
        \frac{\vz - \gamma \vy}{\| \vz - \gamma \vy \| } & \quad \mbox{otherwise}
        \end{cases}.
    \]
    This concludes the proof.
\end{proof}
We use the following theorem to compute the proximal of $g(\vz) = \delta_\mathcal{C}(\vz)$, adapted from \cite[Theorem~6.27]{beck2017first}.

\begin{theorem}[projection onto the intersection of a hyperplane and a box]
\label{thm:prox}
Let $\mathcal{C} = \{ \vx \in \R^n \, | \, \vx^T \vone = K, \vzero \leq \vx \leq \vone \}$ be a set. The proximal operator of an indicator function to the set, $\delta_C$, is given by
\begin{equation}
    \prox_{\gamma \delta_\mathcal{C}} (\vx) = \mathcal{P}_{[\vzero,\vone]} \left( \vx - \overline{\mu} \vone \right)
    \label{eq:OrthProj:Formula}
\end{equation}
where $\mathcal{P}_{[\vzero,\vone]}$ is a projection onto the box $[0,1]^n$ and $\overline{\mu}$ is a solution of the equation $\vone^T \mathcal{P}_{[\vzero,\vone]} \left( \vx - \mu \vone \right) = K$.
\end{theorem}
\begin{proof}
    The orthogonal projection of $\vx$ is the unique solution of 
    \begin{equation}
        \underset{\vz \in \R^n}{\min} \left\lbrace \tfrac{1}{2} \| \vz - \vx \|^2 \, : \, \vone^T \vz = K, \, \vzero \leq \vz \leq \vone \right\rbrace.
        \label{eq:OrthProj}
    \end{equation}
    A Lagrangian of this minimization problem reads 
    \begin{align*}
        \mathcal{L}(\vz,\mu) &= \tfrac{1}{2} \norm{ \vz - \vx }^2 + \mu \left( \vone^T \vz - K\right),
    \end{align*}
    where $\mu \in \R$ is a Lagrange multiplier. It follows from the strong duality that $\overline{\vz}$ is an optimal solution of problem~\eqref{eq:OrthProj} if and only if there exists a dual variable $\mu \in \R$ for which
    \begin{align}
        \vy^\star & \in \argmin_{\vzero \leq \vz \leq \vone} \quad \mathcal{L}(\vz, \overline{\mu}), \label{eq:OrthProj:relation1} \\
        \vone^T & \vz = K.
        \label{eq:OrthProj:relation2}
    \end{align}
    Using the expression for the Lagrangian, the relation~\eqref{eq:OrthProj:relation1} can be equivalently written as $ \overline{\vz} = \mathcal{P}_{[\vzero,\vone]} \left( \vx - \overline{\mu} \vone \right)$. 
    The feasibility condition~\eqref{eq:OrthProj:relation2} takes a form $\vone^T \mathcal{P}_{[\vzero,\vone]} \left( \vx - \overline{\mu} \vone \right) = K$.
\end{proof}

\begin{remark}
    The projection onto the box $[\mathbf{0},\mathbf{1}]$ is simple. It is done component-wise as $\big(\min \{ \max \{ x_{i} , 0\},1\} \big)_{i=1}^{n}$. However, equation~\eqref{eq:OrthProj:Formula} consists of finding a root of the non-increasing~function $\phi(\mu) = \sum\nolimits_{i=1}^{n} \min \{ \max \{ x_{i} - \mu, 0\},1\} - K.$
    Since $\mu \mapsto \min \{ \max \{x_i - \mu , 0 \}, 1 \}$ is a non-increasing function for any i, $\phi$ is a non-increasing function. Its root can be found using the Newton procedure, where derivative is
    \begin{align*}
        \phi'(\mu) = | \mathcal{I} |, \quad
        \mbox{with} \; \mathcal{I} = \left\lbrace i \in \{1, \dots, n\}: 0 \leq x_i - \mu \leq 1 \right\rbrace.
    \end{align*} \vspace{-4mm}
\end{remark}


\begin{algorithm}[t]
    \caption{Image formation from shape coefficients}
    \label{alg:imForm}
    \begin{algorithmic}[1]
    \renewcommand{\algorithmicrequire}{\textbf{Input:}}
    \renewcommand{\algorithmicensure}{\textbf{Output:}}
    \REQUIRE $\vz \in \R^p, \mA \in \R^{m \times n}, \mPsi \in \R^{n \times p}, \vy \in \R^{m}, K$ 
    \ENSURE  $\overline{\vx}$
 \\ \textit{Initialisation} : $\overline{\vx} = \vzero, i=0, s=0$
 \STATE sort the elements of $\vz$ in the descending order, and store its indices as a list $T$
  \WHILE {$\left( s \leq K\right)$ or $\left( i \leq p\right)$}
  \STATE $\hat{\vx} = \overline{\vx} + \mPsi \ve_{T(i)}$
  \IF {$ \| \mA \hat{\vx}  - \vy \| \leq \| \mA \overline{\vx} - \vy \| $ and $ \overline{\vx}^T \mPsi \ve_{T(i)}  \leq 0 $ } \label{alg:imForm:step:overlap}
  \STATE $\overline{\vx} = \hat{\vx}, \quad s = s+1$
  \ENDIF
  \STATE $i = i + 1$
  \ENDWHILE
 \RETURN $\overline{\vx}$
 \end{algorithmic} 
\end{algorithm} 

\vspace{-4mm}
\subsection{Image Formation}\label{sec:Opt:ImForm}
The convex program CoShaRP does not always lead to a binary solution (refer to Fig.~\ref{fig:geometry}). Moreover, if the optimization procedure is terminated early, we may not have a binary solution. Hence, an accurate image formation process is essential to retrieve the target image from the non-binary shape coefficient vector resulting from CoShaRP. We propose the image formation procedure based on sorting of the coefficients. We first sort the coefficients in descending order, and selectively form the image consistent with the measurements. Algorithm~\ref{alg:imForm} enumerates the steps in the image formation process. Here, $\ve_{j}$ is a natural basis vector with non-zero element located at $j^{\text{th}}$ position. To make sure the shapes do not overlap, we also add necessary conditions (see step~\ref{alg:imForm:step:overlap} in Algorithm~\ref{alg:imForm}).

\section{Numerical Experiments}
\label{sec:Results}
In this section, we try to answer questions regarding resolution, sparsity, rotations and the performance under noise using 2D numerical experiments. For all the experiments, we have images of size $1~m \times 1~m$ discretized on $128 \times 128$ pixels, and the tomography matrix has at least $1024$ measurements. The typical tomography setup is shown in Fig.~\ref{fig:setup}. In the CoShaRP performance plots, we generate $100$ different realizations of the target image with given constraints (for examples, size, rotations and repetitions of the shape), and the success rate is measured from the average over all instances. We say an instance is successful if the recovered image is close to the target image (in Euclidean norm).

For all the experiments, we run Algorithm~\ref{alg:split} with $\gamma = 1.2 \sigma$, $\tau = 0.8 \sigma$ with $\sigma = \| \mA \mPsi \|^{-1}$. Moreover, we set $T = 4 p^2$ and $\epsilon = 10^{-6}$. Once we obtain the vector $\overline{\mathbf{z}}$, we run Algorithm~\ref{alg:imForm} to form the image. 

\begin{figure}[t]
    \centering
    \includegraphics[height=0.2\textheight]{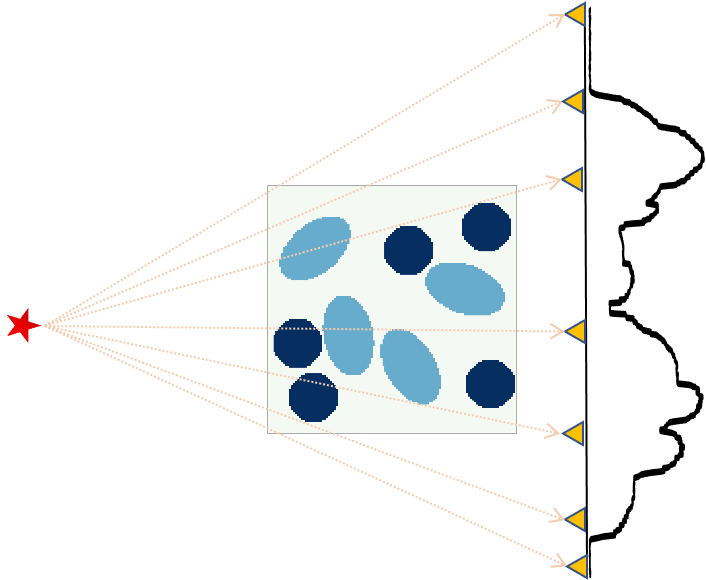} 
    \caption{Single-shot tomography setup. The source (denoted by $\star$) is located at the left of the target image. The detector array consisting of 1024 detectors is located at the right of the image. For reference, we only show 7 detectors (denoted by $\bigtriangledown$). Measurement profile is shown on the right of the detector array.} 
    \label{fig:setup}
\end{figure}


\subsection{Resolution Analysis} 
In this experiment, we estimate the required minimum width of the shape sensed by a single-shot. For simplicity, we consider circular disc of constant intensity with size varying from $1$ to $3168$ pixels. Fig.~\ref{fig:results}(a) shows the performance of CoShaRP against varying sizes of the disc.  As the number of measurements (\ie, detector pixels) is increased, the success rate increases implying that the recovery of even single-pixel shapes is possible with CoShaRP.

\subsection{Invariance with respect to Density and Rotation}
We first look at the success of CoShaRP with multiple repetitions of the shape. We consider a circular disc with size $256$ pixels and the number of repetitions in the image from $1$ to $20$. The top figure in Fig.~\ref{fig:results}(b) shows the performance of CoShaRP with density. This shows that the CoShaRP is insensitive to the number of repetitions of the shapes in the image. Next, we take an ellipsoidal disc with semi-axes $0.2~m$ and $0.08~m$. This image is rotated for $30$ angles making sure that each angle represents a different shape on the $128 \times 128$ pixels. The bottom figure in Fig.~\ref{fig:results}(b) demonstrates the performance of CoShaRP with the number of possible rotations. This implies that the CoShaRP is insensitive to the number of possible rotations of the shapes. 

\begin{figure*}[!htb]
    \centering
    \begin{tabular}{C{.22\textwidth} | C{.22\textwidth}  | C{.25\textwidth}  | C{.25 \textwidth}}
        \toprule
        (a) {\bf Resolution} & (b) {\bf Invariance} & (c) {\bf Shapes} & (d) {\bf Noise} \\ \midrule
        \multirow{2}[1]{=}{ \hspace*{3cm} \\[1ex] \hspace*{1cm} shapes \\ \includegraphics[width=0.2\textwidth]{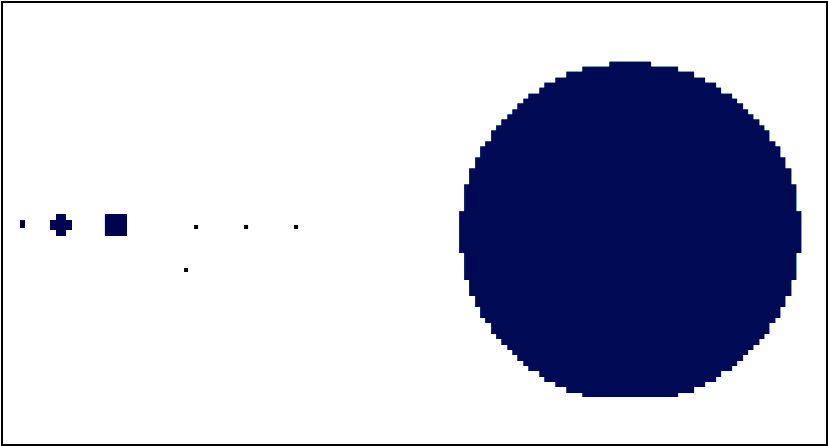} } & \multirow{2}[8]{=}{ \hspace*{1cm} {\bf Density} \\ \hspace*{0.1cm}  \includegraphics[width=0.2\textwidth]{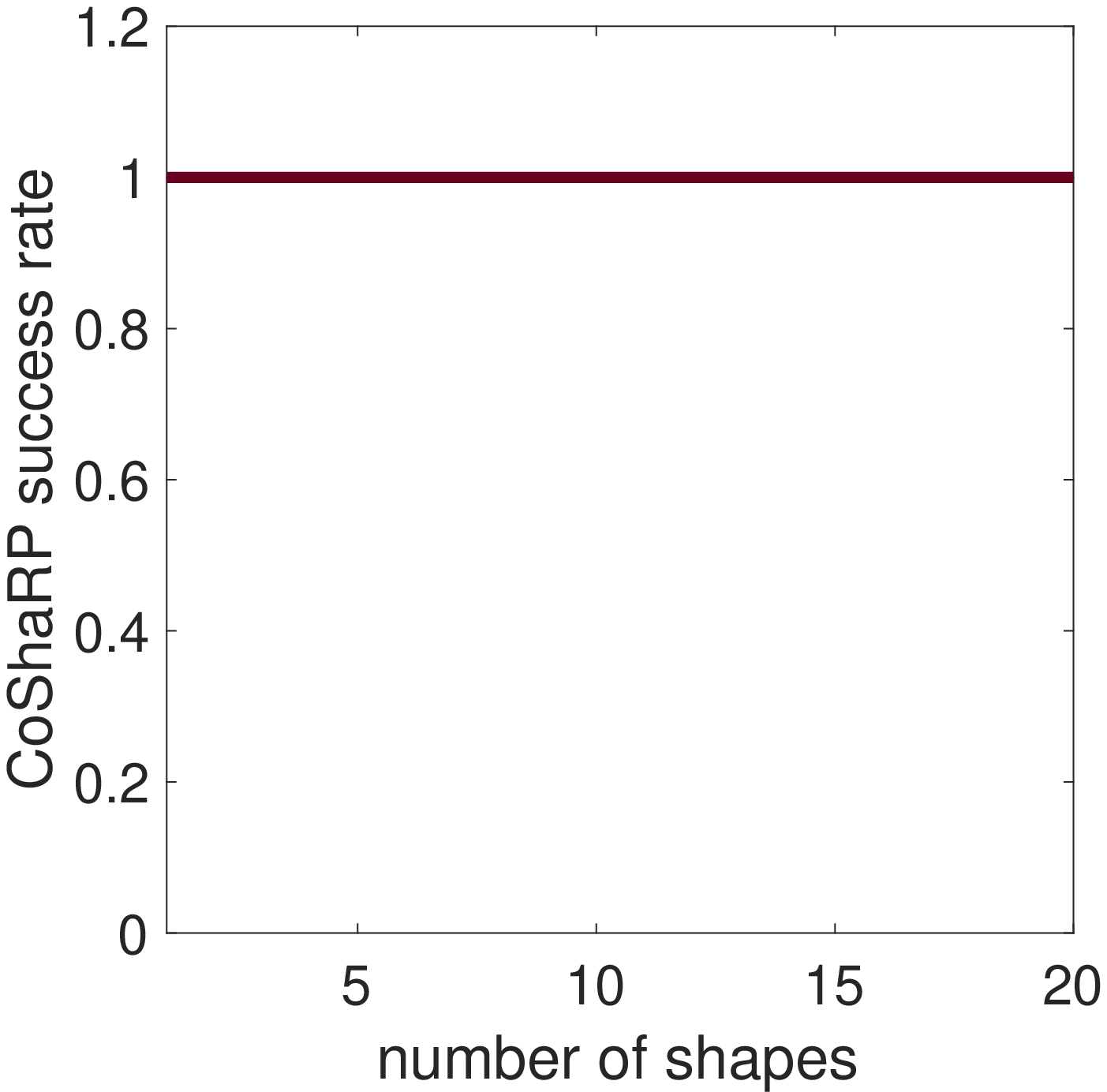}}  & \multirow{2}[8]{=}{ \hspace*{0.2cm} {\bf Non-homogeneous} \\ \hspace*{0.4cm} \includegraphics[width=0.14\textwidth]{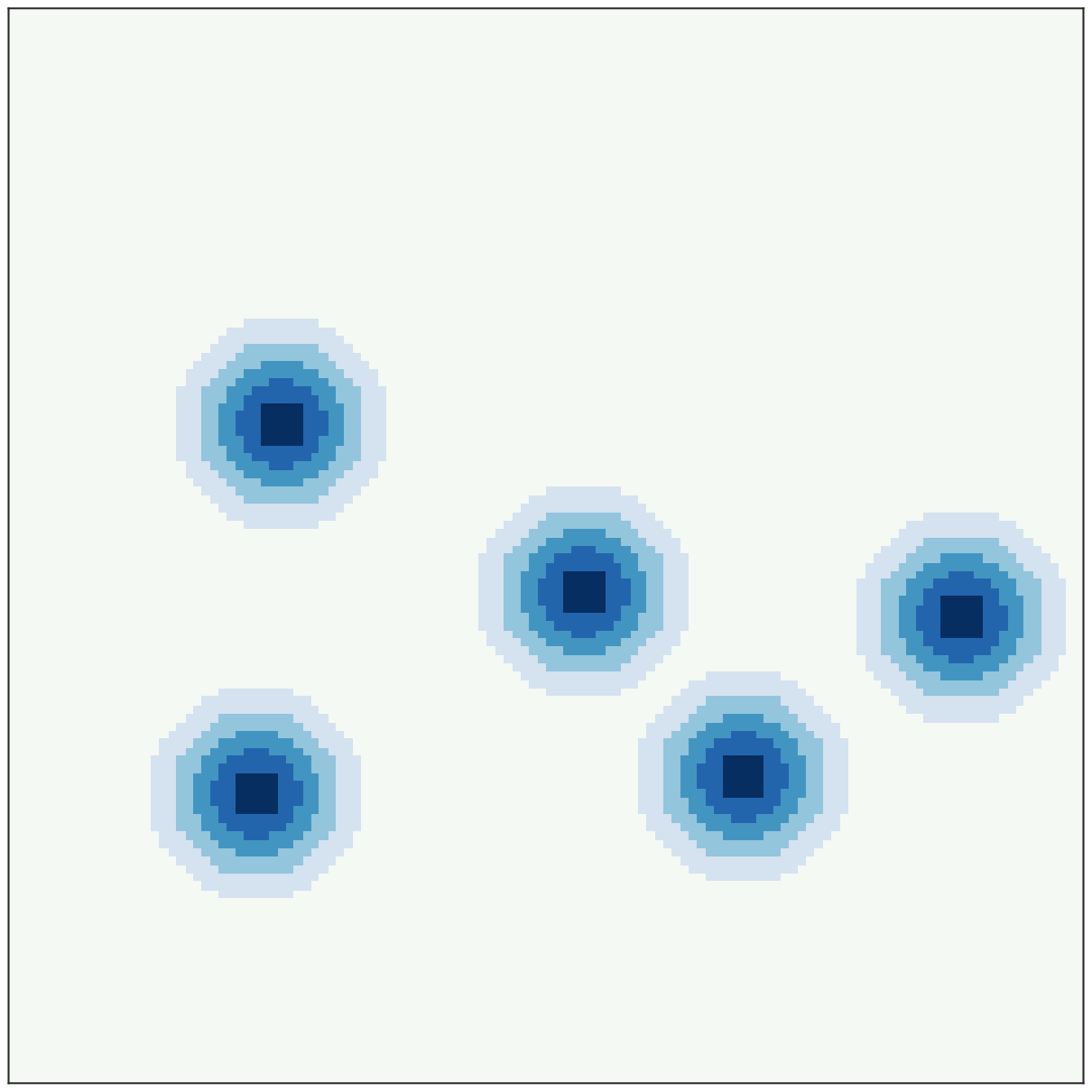} \\ \hspace*{0.4cm} \includegraphics[width=0.14\textwidth]{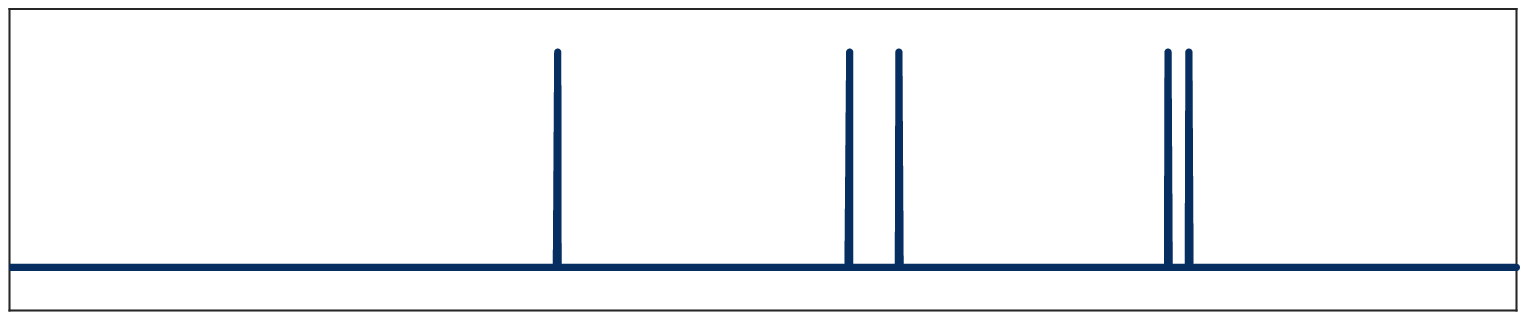}} & $0.1~\%$ SNR  \\
        & & & \includegraphics[width=0.12\textwidth]{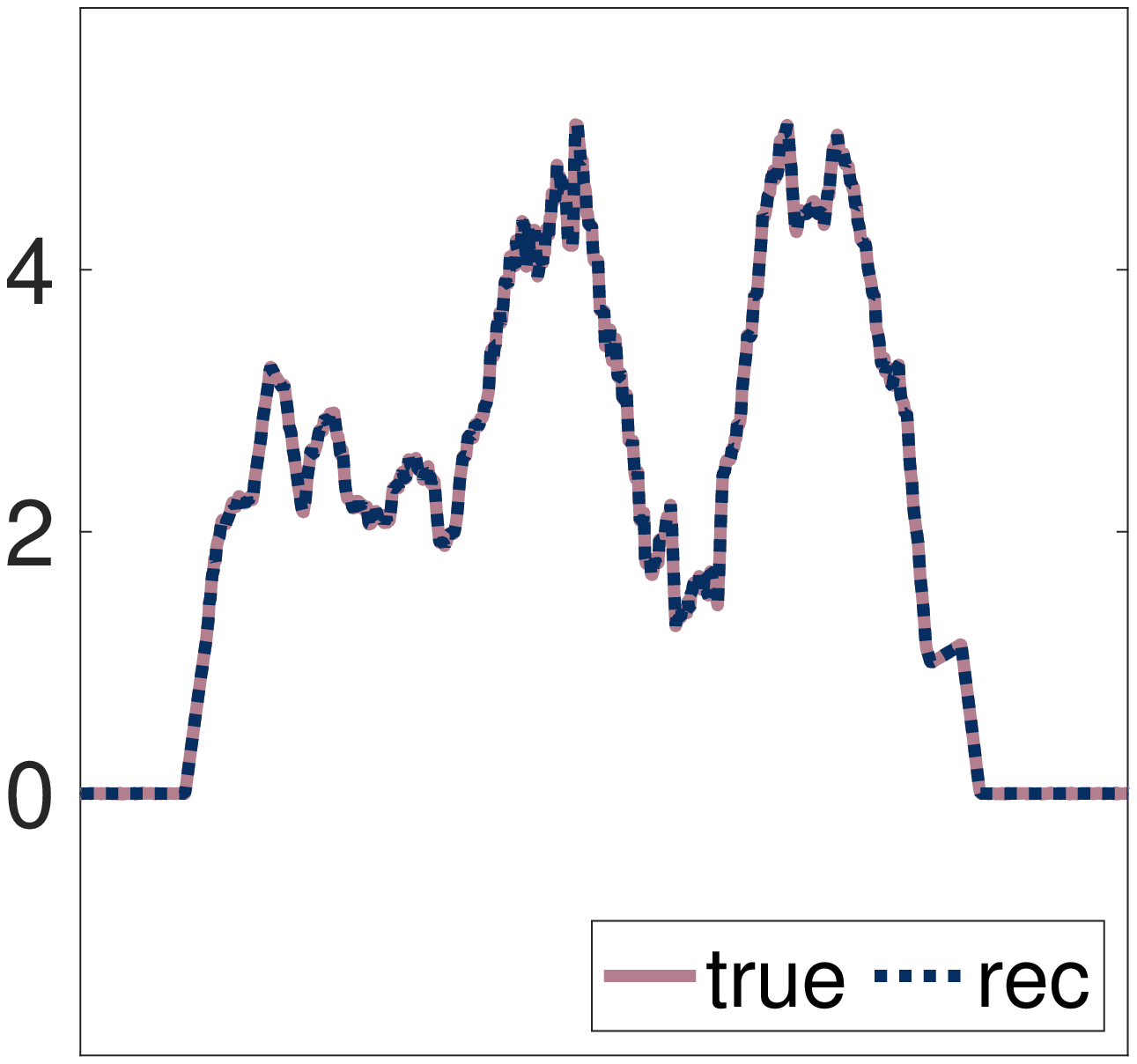} \includegraphics[width=0.12\textwidth]{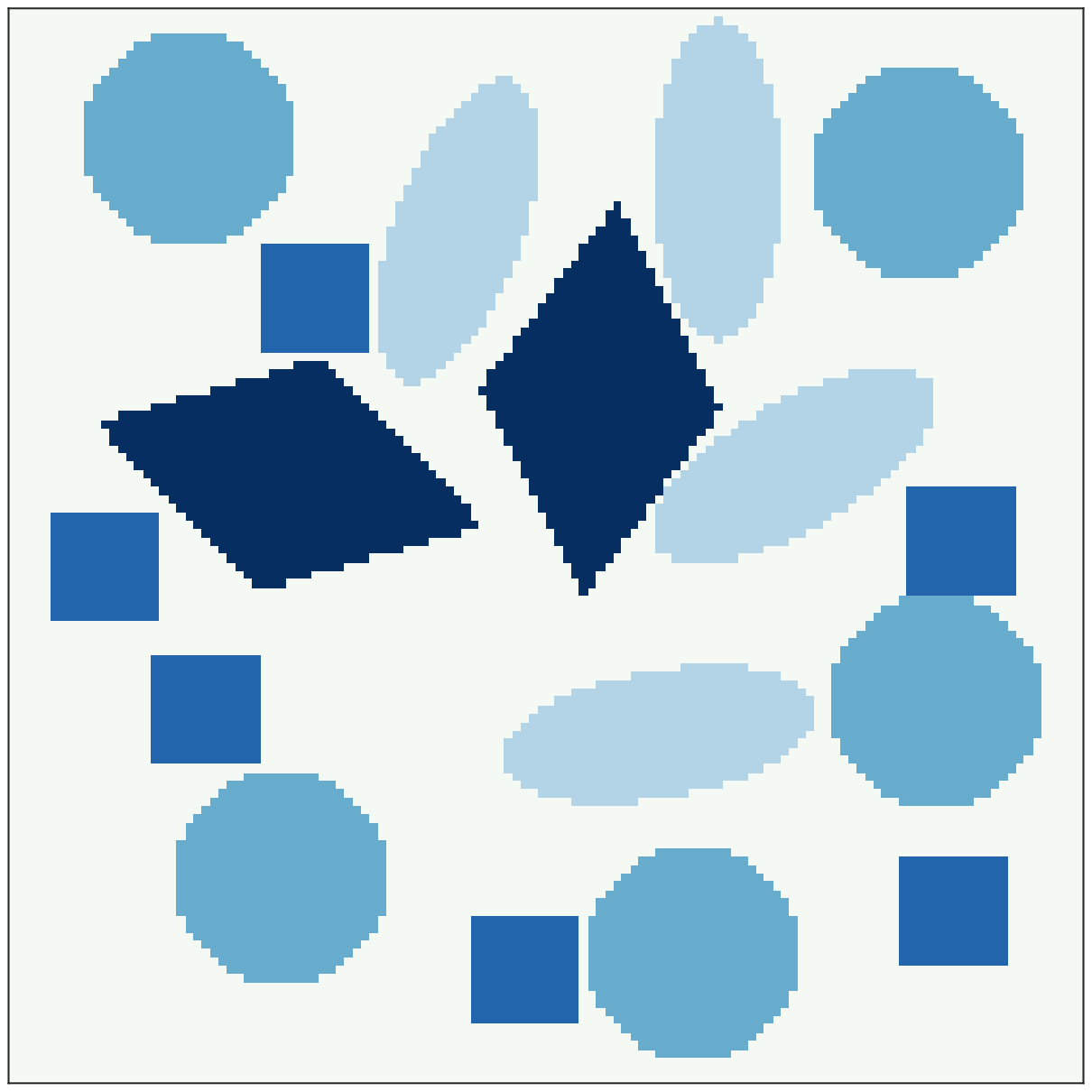} \includegraphics[width=0.12\textwidth]{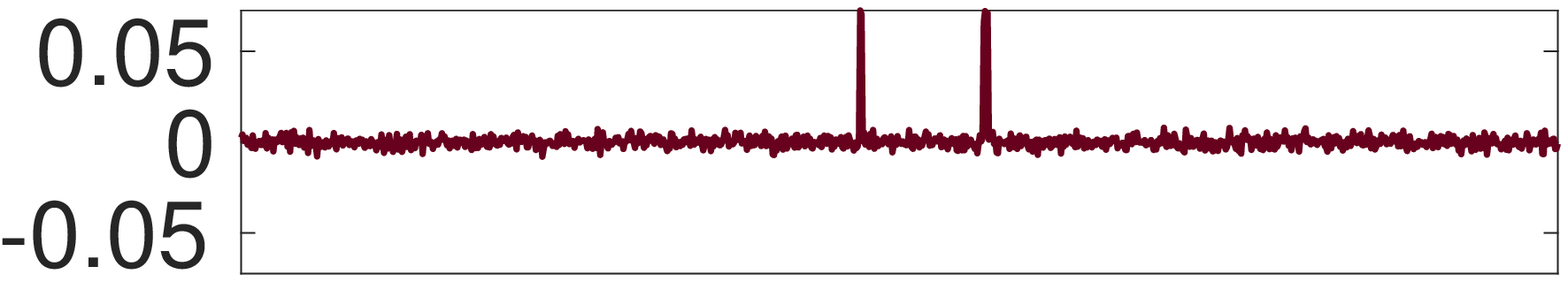} \includegraphics[width=0.12\textwidth]{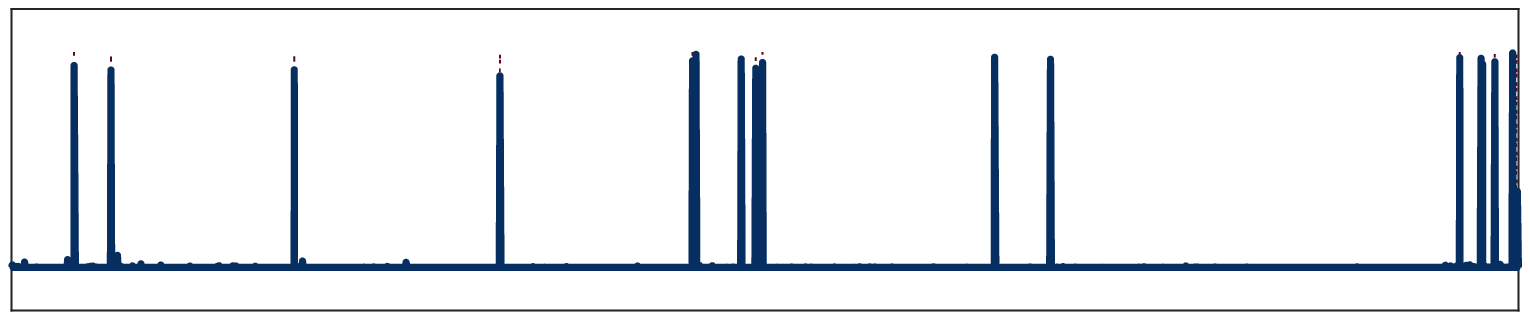} \\
        
         & & & $1~\%$ SNR \\
         \multirow{3}[1]{=}{ \includegraphics[width=0.21\textwidth]{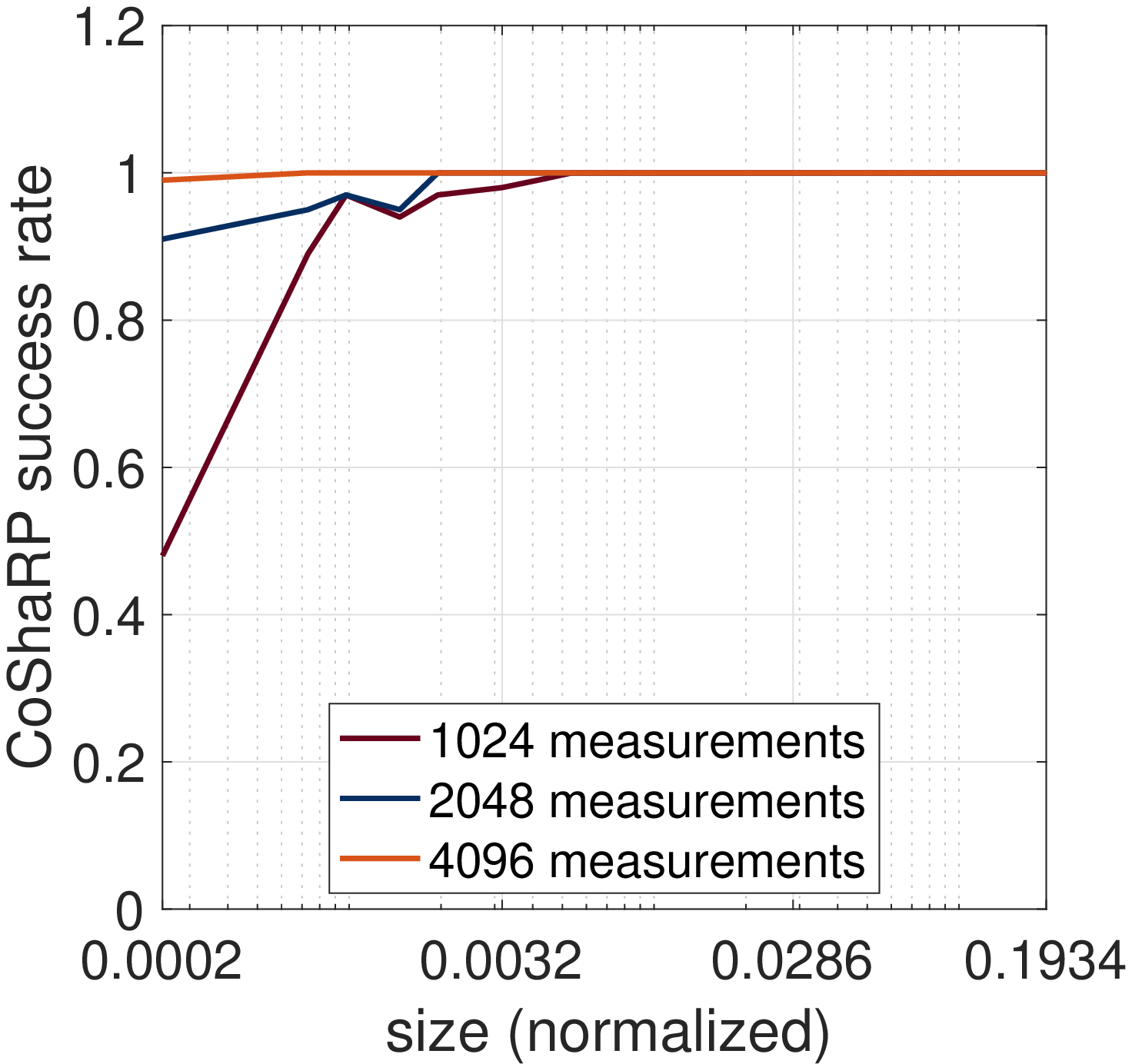} } & \multirow{4}[5]{=}{ \hspace*{1cm} {\bf Rotation} \\ \hspace*{0.1cm}  \includegraphics[width=0.2\textwidth]{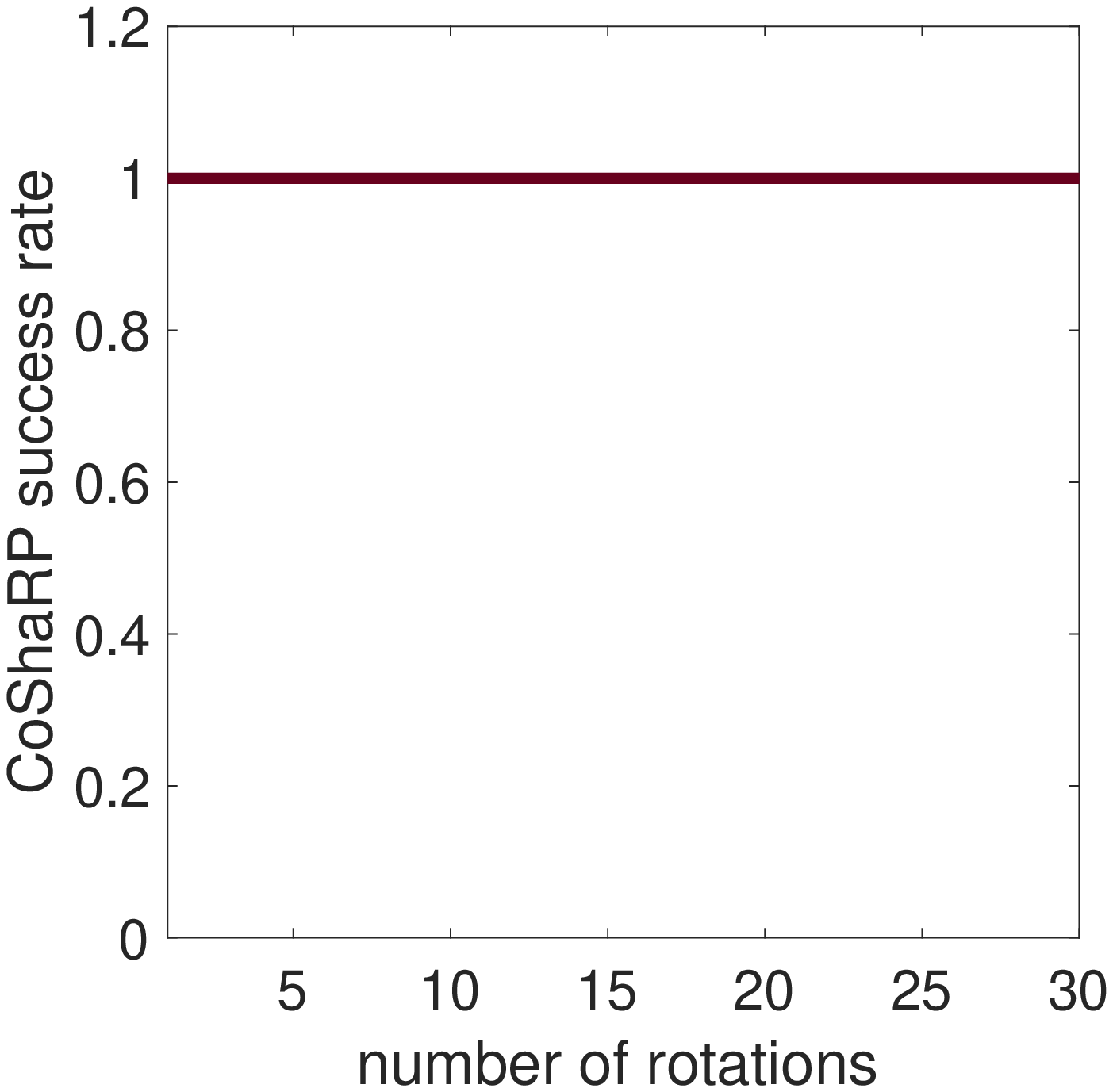}} & \multirow{4}[8]{=}{ \hspace*{0.6cm} {\bf Non-convex} \\ \hspace*{0.4cm} \includegraphics[width=0.14\textwidth]{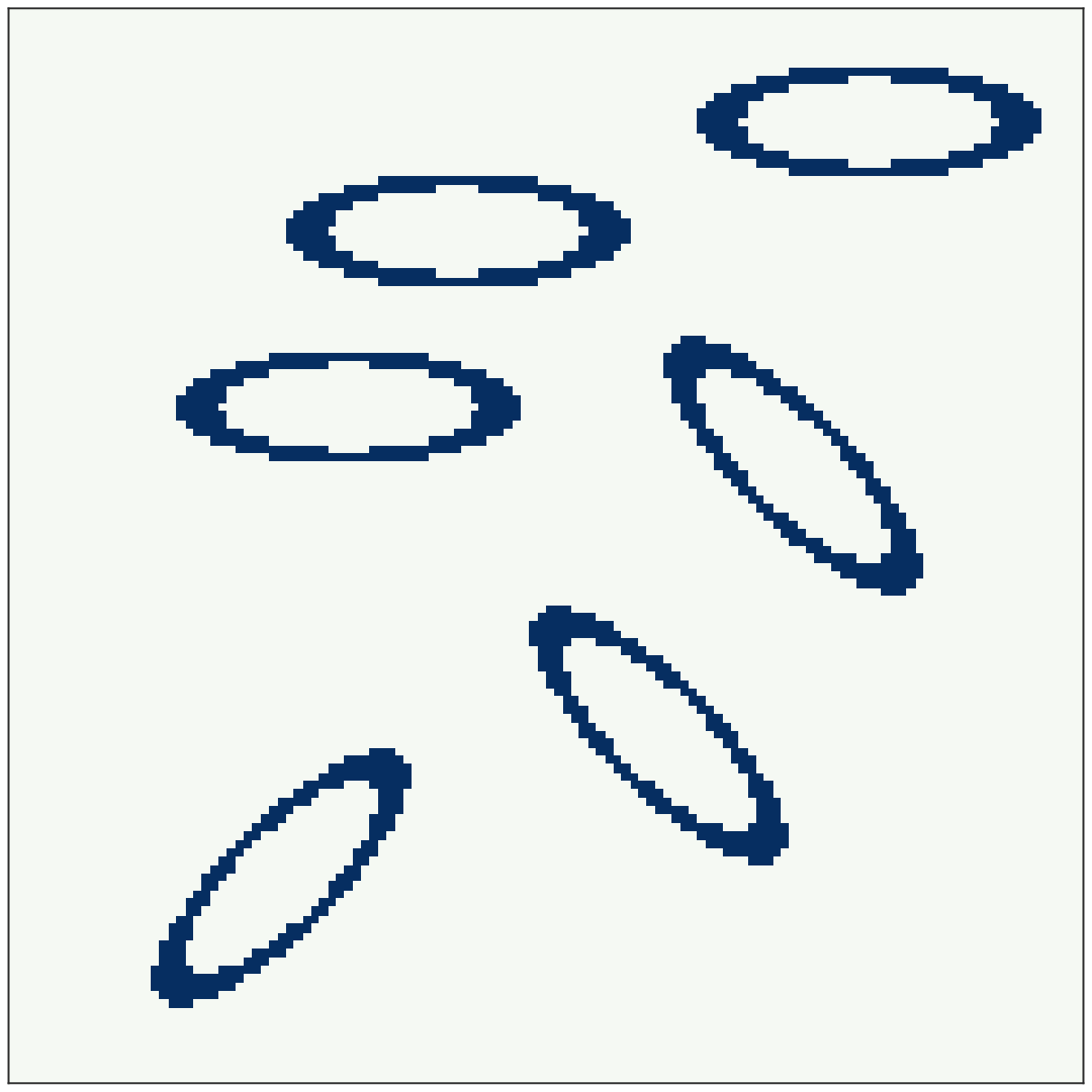} \\ \hspace*{0.4cm} \includegraphics[width=0.14\textwidth]{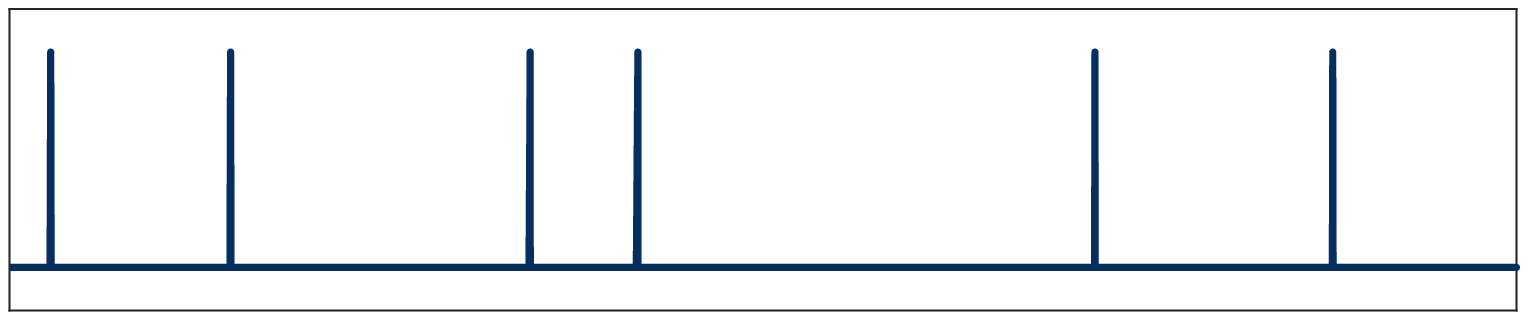}} & \includegraphics[width=0.12\textwidth]{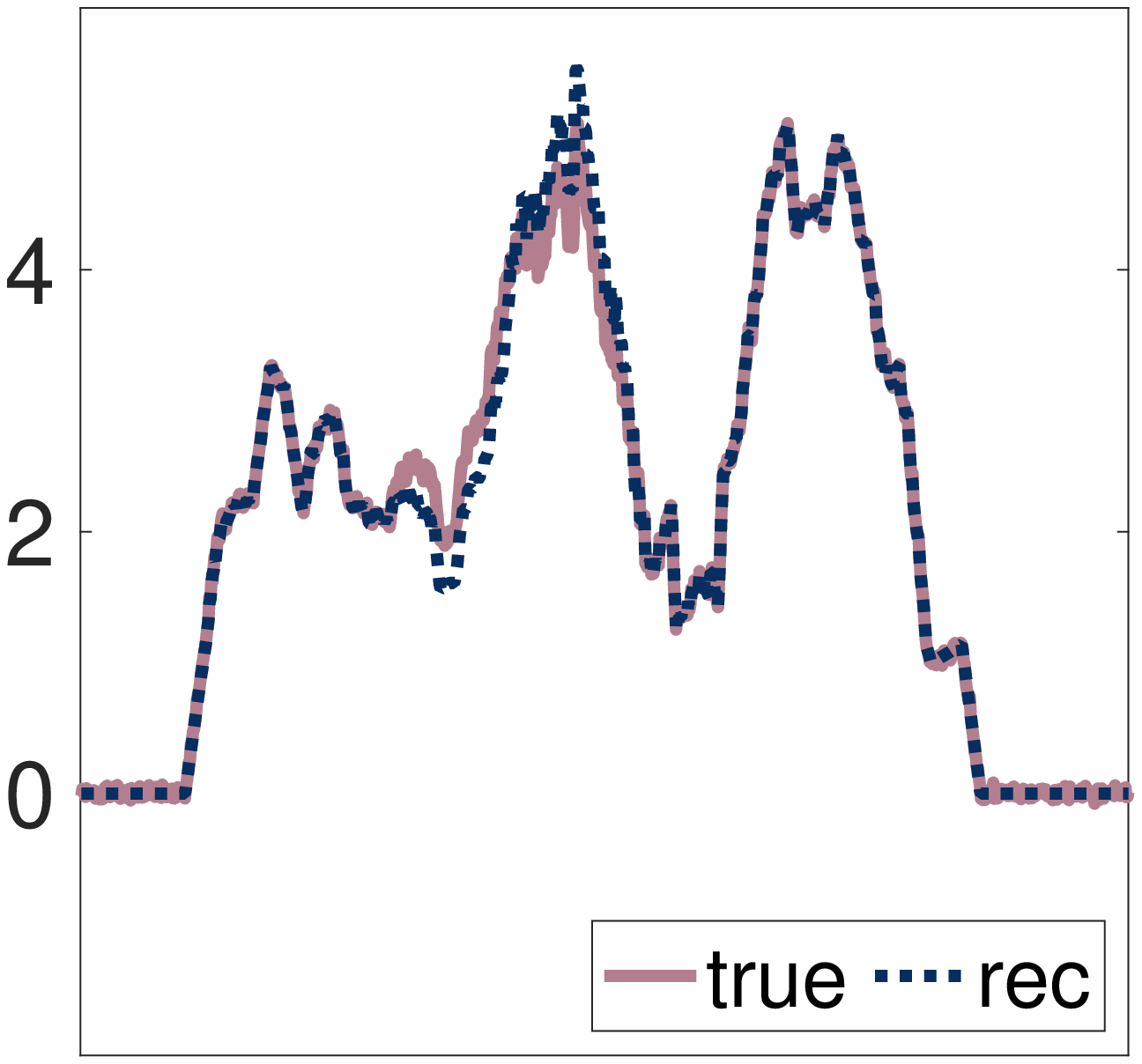} \includegraphics[width=0.12\textwidth]{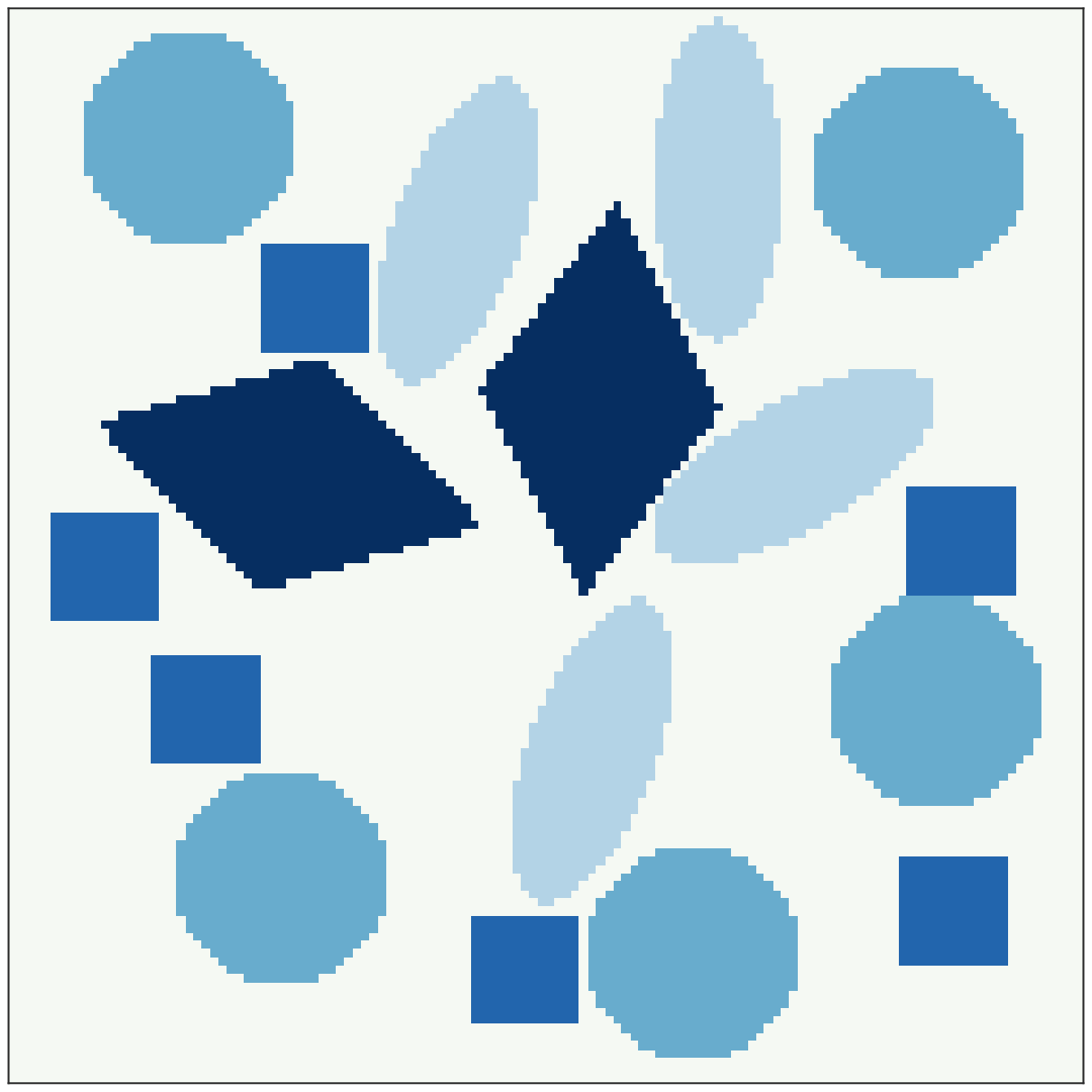}  \includegraphics[width=0.12\textwidth]{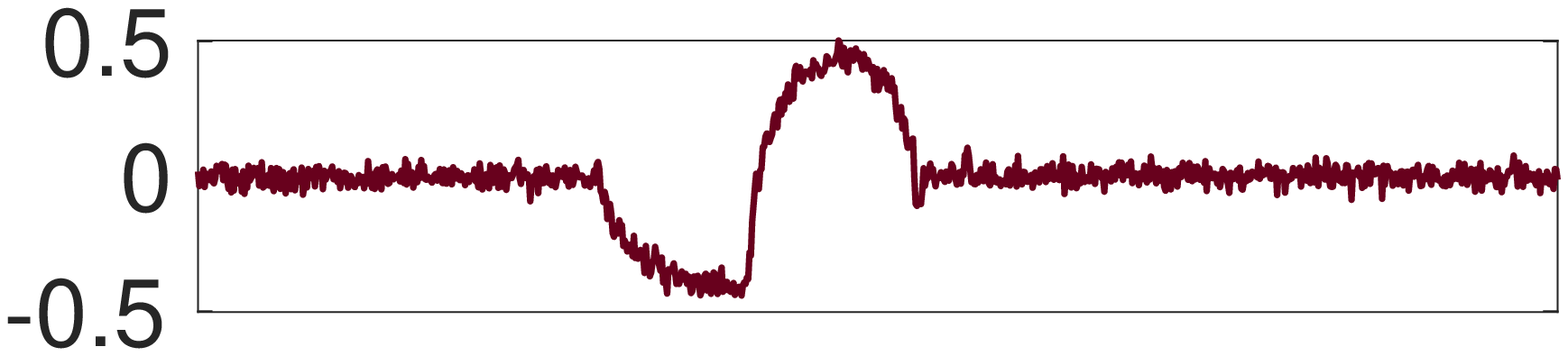}
         \includegraphics[width=0.12\textwidth]{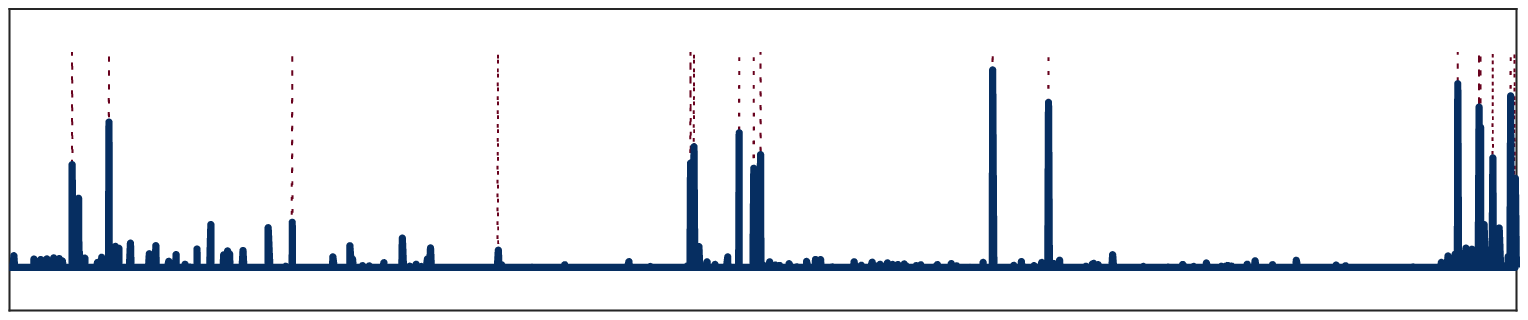} \\
        
         & & & $10~\%$ SNR \\
         & & & \includegraphics[width=0.12\textwidth]{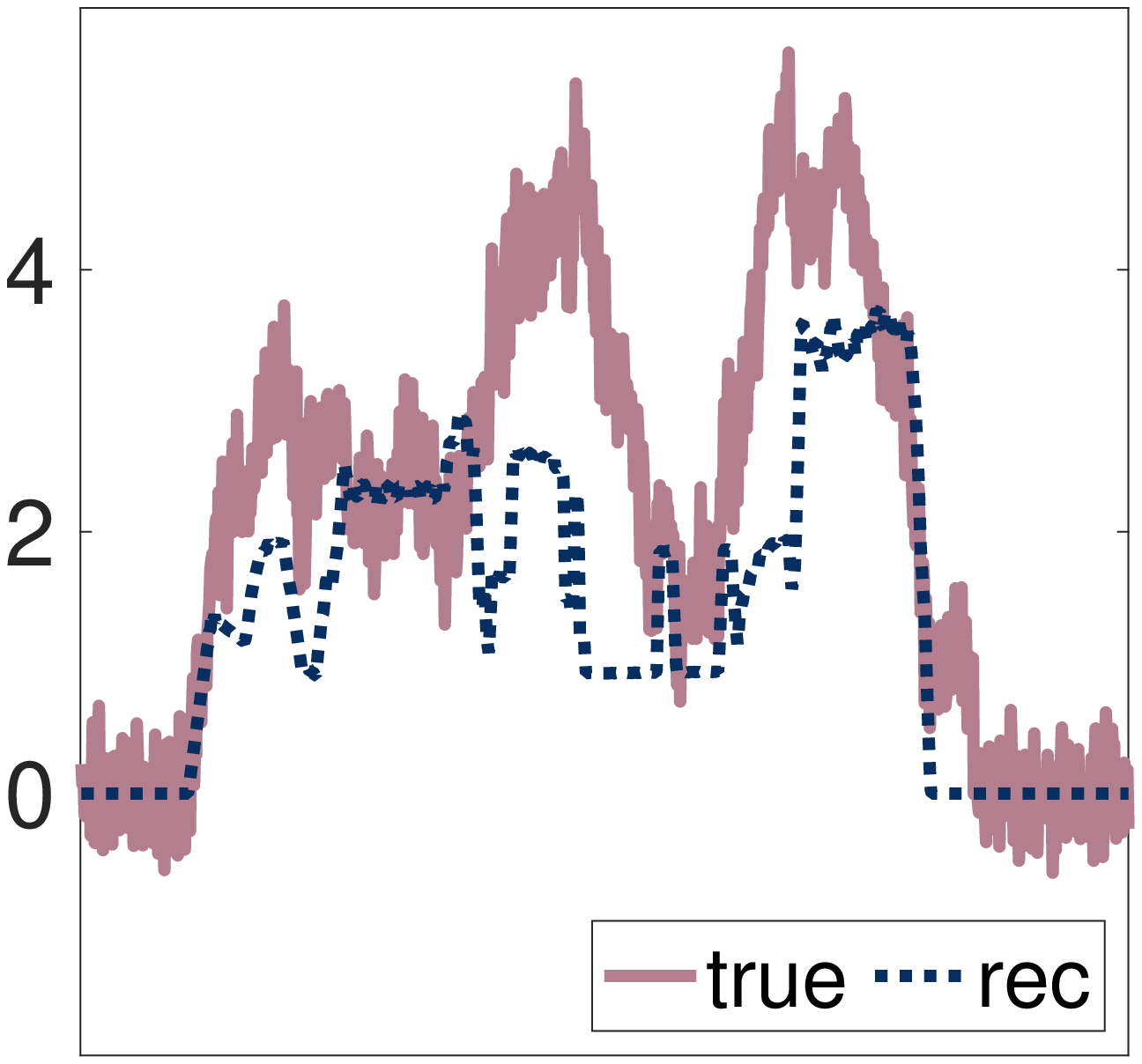} \includegraphics[width=0.12\textwidth]{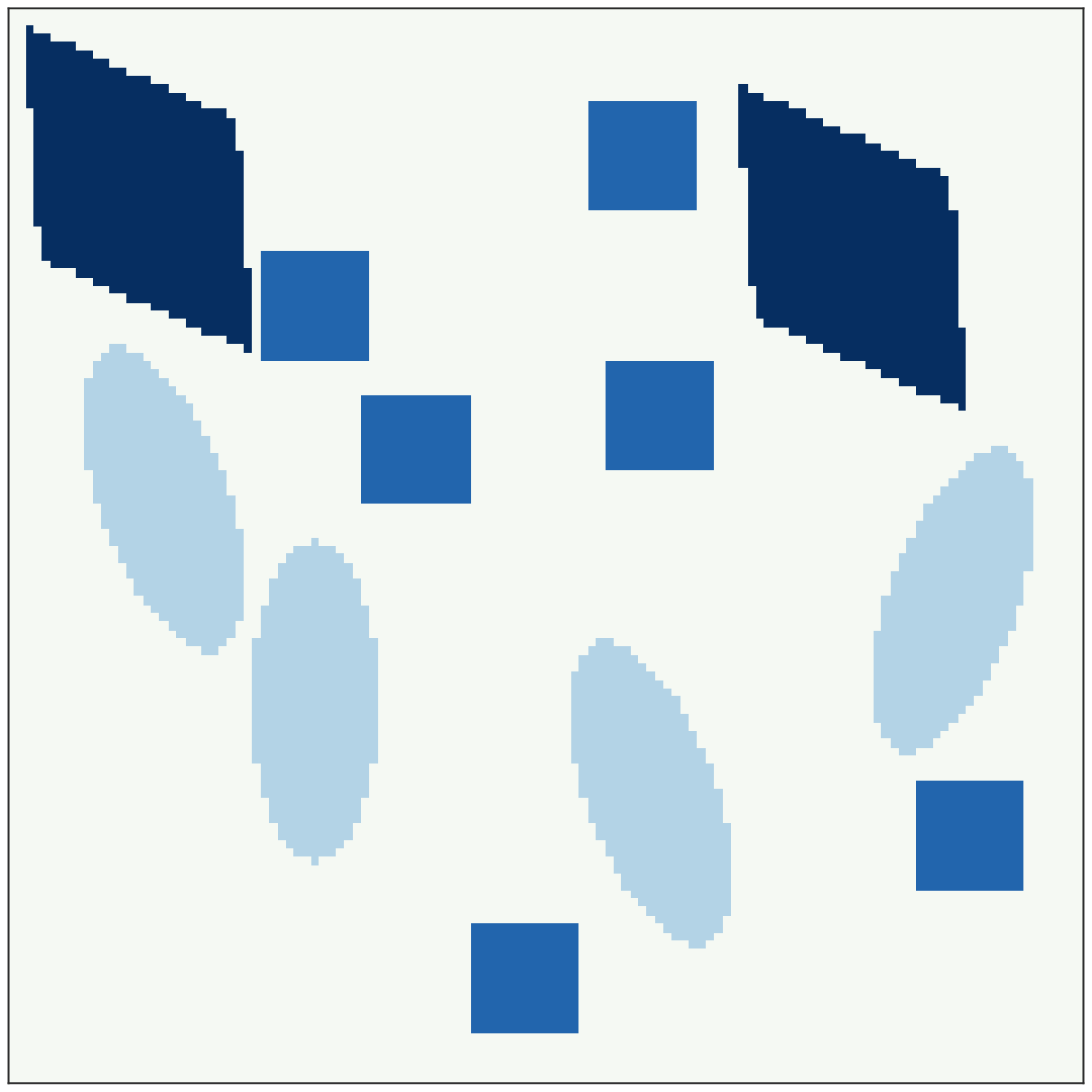}
         \includegraphics[width=0.12\textwidth]{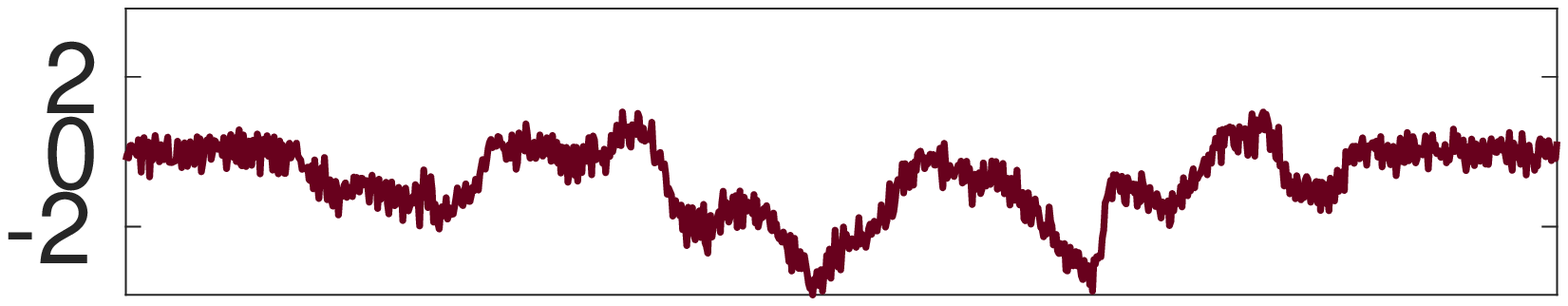}
         \includegraphics[width=0.12\textwidth]{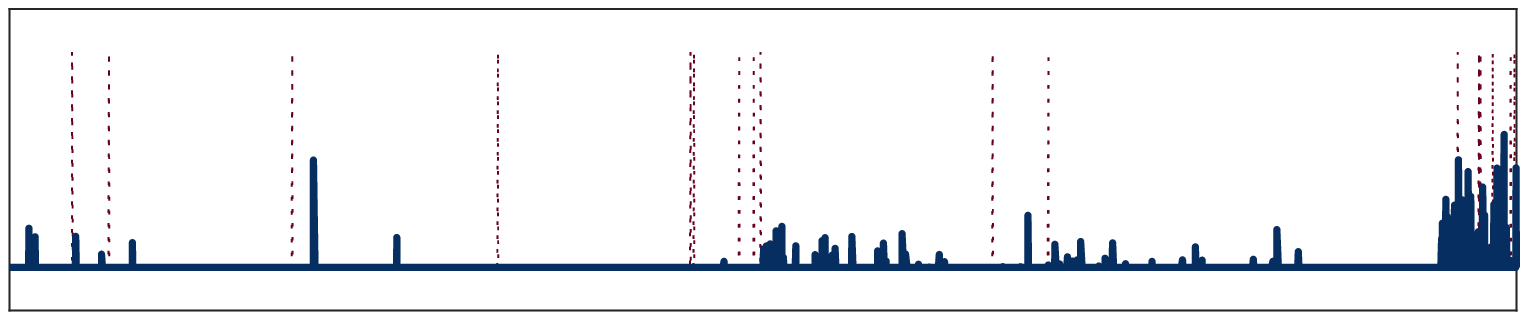} \\
         \bottomrule
    \end{tabular} 
    \caption{Numerical results of CoShaRP.} 
    \label{fig:results}
\end{figure*}

\subsection{Non-homogeneous and Non-convex Shapes}
The top figure in Fig.~\ref{fig:results}(c) provides the reconstruction from CoShaRP for non-homogeneous shape. We consider a circular shape with four different intensities varying radially. The true image consists of 5 repetitions of this shape. The CoShaRP recovers these 5 copies successfully as shown from the shape coefficients (given below the figure). For non-convex shape, we consider ellipsoidal shell with outer axes $0.2~m$ and $0.05~m$, and inner axes $0.15~m$ and $0.03~m$. The bottom figure in Fig.~\ref{fig:results}(c) gives the reconstruction with CoShaRP for a true image with 6 repetitions of the above-mentioned non-convex shape. From these two experiments, we conclude that the CoShaRP can recover the non-homogeneous and non-convex shapes.

\subsection{Measurement Noise}
We consider the true image shown in Fig.~\ref{fig:intro} and consider three noisy scenarios where the Gaussian noise of strength $0.1\%$, $1\%$ and $10\%$ is added to 1024 measurements. In Fig.~\ref{fig:results}(d), we plot the measurements (on the left) and the reconstructed image (on the right) for the above-mentioned noise values. In the measurements plots, the true noisy data is mentioned as `\textit{true}', while the forward projected data from the reconstructed image is denoted by `\textit{rec}'. The plots below them show the difference between the two. From Fig.~\ref{fig:results}(d), it is evident that the CoShaRP is stable till $1\%$ noise, while fails for extremely noisy measurements. 


\section{Conclusions}\label{sec:Conclusion}

We introduced a \textit{single-shot tomographic shape sensing} problem that aims to recover shapes from a single cone-beam projection. To solve this problem, we develop a convex program CoShaRP. CoShaRP is novel in the sense that the simplex-type constraint enables sharp recovery results from extremely under-determined single-shot tomographic projections. Moreover, we propose a primal-dual algorithm to find an approximately optimal solution to CoShaRP quickly. The numerical results demonstrate that 
\begin{enumerate}[topsep=5pt,itemsep=0ex,partopsep=1ex,parsep=1ex]
	\item  the resolution limit to sense the shape depends on the number of measurements,
	\item CoShaRP is insensitive to the number of repetitions of the shape and the number of possible rotations of the shape, 
	\item CoShaRP can sense the non-homogeneous as well as non-convex shapes,
	\item CoShaRP tolerates only a moderate amount of measurement noise. 
\end{enumerate}

\noindent The limitations of CoShaRP are as follows: (\emph{i}) The roto-translations of the shapes must be included in the dictionary for the exact recovery of the target image. This inclusion requirement makes CoShaRP a computationally expensive approach due to the large dictionary size.
(\emph{ii}) CoShaRP also requires the correct knowledge of shapes and their intensity. If the shape is not known accurately, the CoShaRP may fail.
(\emph{iii}) CoShaRP relies on the knowledge of total number of shapes in the target image. If unknown, its estimation may be a costly procedure due to repeated solving of CoShaRP for various estimates.

\section*{Acknowledgments}
The authors thank Nick Luiken for stimulating discussions. This work was supported by the Dutch Research Council (grant no. OCENW.XS.039).


\begin{thebibliography}{10}

\bibitem{Aghasi2013}
Alireza Aghasi and Justin Romberg.
\newblock Sparse shape reconstruction.
\newblock {\em {SIAM} Journal on Imaging Sciences}, 6(4):2075--2108, January
  2013.

\bibitem{Aghasi2015}
Alireza Aghasi and Justin Romberg.
\newblock Convex cardinal shape composition.
\newblock {\em {SIAM} Journal on Imaging Sciences}, 8(4):2887--2950, January
  2015.

\bibitem{batenburg2011dart}
K.~J. Batenburg and J.~Sijbers.
\newblock {DART}: A practical reconstruction algorithm for discrete tomography.
\newblock {\em {IEEE} Transactions on Image Processing}, 20(9):2542--2553,
  September 2011.

\bibitem{beck2017first}
Amir Beck.
\newblock {\em First-Order Methods in Optimization}.
\newblock Society for Industrial and Applied Mathematics, October 2017.

\bibitem{Bicer2017}
Tekin Bicer, Do{\u{g}}a G\"{u}rsoy, Vincent~De Andrade, Rajkumar Kettimuthu,
  William Scullin, Francesco~De Carlo, and Ian~T. Foster.
\newblock Trace: a high-throughput tomographic reconstruction engine for
  large-scale datasets.
\newblock {\em Advanced Structural and Chemical Imaging}, 3(1), January 2017.

\bibitem{boyd2004convex}
Stephen Boyd, Stephen~P Boyd, and Lieven Vandenberghe.
\newblock {\em Convex optimization}.
\newblock Cambridge university press, 2004.

\bibitem{Boyd2010}
Stephen Boyd, Neal Parikh, Eric Chu, Borja Peleato, and Jonathan Eckstein.
\newblock Distributed optimization and statistical learning via the alternating
  direction method of multipliers.
\newblock {\em Foundations and Trends{\textregistered} in Machine Learning},
  3(1):1--122, 2010.

\bibitem{Chambolle2010}
Antonin Chambolle and Thomas Pock.
\newblock A first-order primal-dual algorithm for convex problems
  with~applications to imaging.
\newblock {\em Journal of Mathematical Imaging and Vision}, 40(1):120--145,
  December 2010.

\bibitem{Combettes2011}
Patrick~L. Combettes and Jean-Christophe Pesquet.
\newblock Proximal splitting methods in signal processing.
\newblock In {\em Springer Optimization and Its Applications}, pages 185--212.
  Springer New York, 2011.

\bibitem{defrise1994cone}
M.~Defrise and R.~Clack.
\newblock A cone-beam reconstruction algorithm using shift-variant filtering
  and cone-beam backprojection.
\newblock {\em {IEEE} Transactions on Medical Imaging}, 13(1):186--195, March
  1994.

\bibitem{delaney1998globally}
A.H. Delaney and Y.~Bresler.
\newblock Globally convergent edge-preserving regularized reconstruction: an
  application to limited-angle tomography.
\newblock {\em {IEEE} Transactions on Image Processing}, 7(2):204--221, 1998.

\bibitem{Frikel2013}
J\"{u}rgen Frikel.
\newblock Sparse regularization in limited angle tomography.
\newblock {\em Applied and Computational Harmonic Analysis}, 34(1):117--141,
  January 2013.

\bibitem{Hantke2014}
Max~F. Hantke, Dirk Hasse, Filipe R. N.~C. Maia, Tomas Ekeberg, Katja John,
  Martin Svenda, N.~Duane Loh, Andrew~V. Martin, Nicusor Timneanu, Daniel S.~D.
  Larsson, Gijs van~der Schot, Gunilla~H. Carlsson, Margareta Ingelman, Jakob
  Andreasson, Daniel Westphal, Mengning Liang, Francesco Stellato, Daniel~P.
  DePonte, Robert Hartmann, Nils Kimmel, Richard~A. Kirian, M.~Marvin Seibert,
  Kerstin M\"{u}hlig, Sebastian Schorb, Ken Ferguson, Christoph Bostedt,
  Sebastian Carron, John~D. Bozek, Daniel Rolles, Artem Rudenko, Sascha Epp,
  Henry~N. Chapman, Anton Barty, Janos Hajdu, and Inger Andersson.
\newblock High-throughput imaging of heterogeneous cell organelles with an
  x-ray laser.
\newblock {\em Nature Photonics}, 8(12):943--949, November 2014.

\bibitem{Karp1972}
Richard~M. Karp.
\newblock Reducibility among combinatorial problems.
\newblock In {\em Complexity of Computer Computations}, pages 85--103. Springer
  {US}, 1972.

\bibitem{Kudo1998}
Hiroyuki Kudo, Fr{\'{e}}d{\'{e}}ric Noo, and Michel Defrise.
\newblock Cone-beam filtered-backprojection algorithm for truncated helical
  data.
\newblock {\em Physics in Medicine and Biology}, 43(10):2885--2909, October
  1998.

\bibitem{Laanait2017}
Nouamane Laanait, Wittawat Saenrang, Hua Zhou, Chang-Beom Eom, and Zhan Zhang.
\newblock Dynamic x-ray diffraction imaging of the ferroelectric response in
  bismuth ferrite.
\newblock {\em Advanced Structural and Chemical Imaging}, 3(1), March 2017.

\bibitem{Pelt2016}
Daniël~M. Pelt and Vincent~De Andrade.
\newblock Improved tomographic reconstruction of large-scale real-world data by
  filter optimization.
\newblock {\em Advanced Structural and Chemical Imaging}, 2(1), December 2016.

\bibitem{Persson2001}
M~Persson, D~Bone, and H~Elmqvist.
\newblock Total variation norm for three-dimensional iterative reconstruction
  in limited view angle tomography.
\newblock {\em Physics in Medicine and Biology}, 46(3):853--866, February 2001.

\bibitem{Scheres2012}
Sjors~H.W. Scheres.
\newblock {RELION}: Implementation of a bayesian approach to cryo-{EM}
  structure determination.
\newblock {\em Journal of Structural Biology}, 180(3):519--530, December 2012.

\bibitem{Sidky2008}
Emil~Y Sidky and Xiaochuan Pan.
\newblock Image reconstruction in circular cone-beam computed tomography by
  constrained, total-variation minimization.
\newblock {\em Physics in Medicine and Biology}, 53(17):4777--4807, August
  2008.

\bibitem{Spence2012}
J~C~H Spence, U~Weierstall, and H~N Chapman.
\newblock X-ray lasers for structural and dynamic biology.
\newblock {\em Reports on Progress in Physics}, 75(10):102601, September 2012.

\bibitem{Zanaga2016}
Daniele Zanaga, Folkert Bleichrodt, Thomas Altantzis, Naomi Winckelmans,
  Willem~Jan Palenstijn, Jan Sijbers, Bart de~Nijs, Marijn~A. van Huis, Ana
  S{\'{a}}nchez-Iglesias, Luis~M. Liz-Marz{\'{a}}n, Alfons van Blaaderen,
  K.~Joost Batenburg, Sara Bals, and Gustaaf~Van Tendeloo.
\newblock Quantitative 3d analysis of huge nanoparticle assemblies.
\newblock {\em Nanoscale}, 8(1):292--299, 2016.

\bibitem{Zhang2008}
X.~Zhang, E.~Settembre, C.~Xu, P.~R. Dormitzer, R.~Bellamy, S.~C. Harrison, and
  N.~Grigorieff.
\newblock Near-atomic resolution using electron cryomicroscopy and
  single-particle reconstruction.
\newblock {\em Proceedings of the National Academy of Sciences},
  105(6):1867--1872, January 2008.

\bibitem{Zhang2010}
Xiaoqun Zhang, Martin Burger, and Stanley Osher.
\newblock A unified primal-dual algorithm framework based on~bregman iteration.
\newblock {\em Journal of Scientific Computing}, 46(1):20--46, August 2010.

\bibitem{Zhu2020}
Wenjuan Zhu, Wenbo Ma, Yirong Su, Zeng Chen, Xinya Chen, Yaoguang Ma, Lizhong
  Bai, Wenge Xiao, Tianyu Liu, Haiming Zhu, Xiaofeng Liu, Huafeng Liu, Xu~Liu,
  and Yang~(Michael) Yang.
\newblock Low-dose real-time x-ray imaging with nontoxic double perovskite
  scintillators.
\newblock {\em Light: Science {\&} Applications}, 9(1), June 2020.

\bibitem{Zou2004}
Yu~Zou and Xiaochuan Pan.
\newblock Image reconstruction on {PI}-lines by use of filtered backprojection
  in helical cone-beam {CT}.
\newblock {\em Physics in Medicine and Biology}, 49(12):2717--2731, June 2004.

\end{thebibliography}
\end{document}